\documentclass{article}
\usepackage[T1]{fontenc}
\usepackage[latin9]{inputenc}
\usepackage{array}
\usepackage{multirow}
\usepackage{amsthm}
\usepackage{amsmath}
\usepackage{amssymb}
\usepackage{graphicx}
\usepackage{xargs}[2008/03/08]
\usepackage[unicode=true,pdfusetitle,
 bookmarks=true,bookmarksnumbered=false,bookmarksopen=false,
 breaklinks=false,pdfborder={0 0 1},backref=false,colorlinks=false]
 {hyperref}
\usepackage{breakurl}

\makeatletter

\providecommand{\tabularnewline}{\\}

\theoremstyle{plain}
\newtheorem{thm}{\protect\theoremname}
\theoremstyle{definition}
\newtheorem{defn}[thm]{\protect\definitionname}
\theoremstyle{plain}
\newtheorem{prop}[thm]{\protect\propositionname}
\ifx\proof\undefined
\newenvironment{proof}[1][\protect\proofname]{\par
\normalfont\topsep6\p@\@plus6\p@\relax
\trivlist
\itemindent\parindent
\item[\hskip\labelsep
\scshape
#1]\ignorespaces
}{%
\endtrivlist\@endpefalse
}
\providecommand{\proofname}{Proof}
\fi
\theoremstyle{plain}
\newtheorem{lem}[thm]{\protect\lemmaname}

\usepackage[accepted]{icml2014}
\usepackage{xcolor}
\usepackage{times}

\makeatother

\providecommand{\definitionname}{Definition}
\providecommand{\lemmaname}{Lemma}
\providecommand{\propositionname}{Proposition}
\providecommand{\theoremname}{Theorem}

\begin{document}

\twocolumn[
\icmltitle{Bayesian Nonparametric Multilevel Clustering with Group-Level Contexts}
\icmlauthor{Vu Nguyen$^1$}{tvnguye@deakin.edu.au} 
\icmlauthor{Dinh Phung$^1$}{dinh.phung@deakin.edu.au} 
\icmlauthor{XuanLong Nguyen$^2$}{xuanlong@umich.edu} 
\icmlauthor{Svetha Venkatesh$^1$}{svetha.venkatesh@deakin.edu.au} 
\icmlauthor{Hung Hai Bui$^3$}{bui.h.hung@gmail.com}
\icmladdress{$^1$Center for Pattern Recognition and Data Analytics (PRaDA), Deakin University, Australia}
\vspace{-0.06in}
\icmladdress{$^2$Department of Statistics, University of Michigan, Ann Arbor, USA}
\vspace{-0.06in}
\icmladdress{$^3$Laboratory for Natural Language Understanding, Nuance Communications, Sunnyvale, USA}
\icmlkeywords{Bayesian nonparametric, Dirichlet process}
\vskip 0.3in
]
\begin{abstract}
We present a Bayesian nonparametric framework for multilevel clustering
which utilizes group-level context information to simultaneously discover
low-dimensional structures of the group contents and partitions groups
into clusters. Using the Dirichlet process as the building block,
our model constructs a product base-measure with a nested structure
to accommodate content and context observations at multiple levels.
The proposed model possesses properties that link the nested Dirichlet
processes (nDP) and the Dirichlet process mixture models (DPM) in
an interesting way: integrating out all contents results in the DPM
over contexts, whereas integrating out group-specific contexts results
in the nDP mixture over content variables. We provide a Polya-urn
view of the model and an efficient collapsed Gibbs inference procedure.
Extensive experiments on real-world datasets demonstrate the advantage
of utilizing context information via our model in both text and image
domains.
\end{abstract}
\newcommand{\sidenote}[1]{\marginpar{\small \emph{\color{Medium}#1}}}

\global\long\def\se{\hat{\text{se}}}

\global\long\def\interior{\text{int}}

\global\long\def\boundary{\text{bd}}

\global\long\def\ML{\textsf{ML}}

\global\long\def\GML{\mathsf{GML}}

\global\long\def\HMM{\mathsf{HMM}}

\global\long\def\support{\text{supp}}

\global\long\def\new{\text{*}}

\global\long\def\stir{\text{Stirl}}

\global\long\def\mA{\mathcal{A}}

\global\long\def\mB{\mathcal{B}}

\global\long\def\mF{\mathcal{F}}

\global\long\def\mK{\mathcal{K}}

\global\long\def\mH{\mathcal{H}}

\global\long\def\mX{\mathcal{X}}

\global\long\def\mZ{\mathcal{Z}}

\global\long\def\mS{\mathcal{S}}

\global\long\def\Ical{\mathcal{I}}

\global\long\def\mT{\mathcal{T}}

\global\long\def\Pcal{\mathcal{P}}

\global\long\def\dist{d}

\global\long\def\HX{\entro\left(X\right)}
 \global\long\def\entropyX{\HX}

\global\long\def\HY{\entro\left(Y\right)}
 \global\long\def\entropyY{\HY}

\global\long\def\HXY{\entro\left(X,Y\right)}
 \global\long\def\entropyXY{\HXY}

\global\long\def\mutualXY{\mutual\left(X;Y\right)}
 \global\long\def\mutinfoXY{\mutualXY}

\global\long\def\given{\mid}

\global\long\def\gv{\given}

\global\long\def\goto{\rightarrow}

\global\long\def\asgoto{\stackrel{a.s.}{\longrightarrow}}

\global\long\def\pgoto{\stackrel{p}{\longrightarrow}}

\global\long\def\dgoto{\stackrel{d}{\longrightarrow}}

\global\long\def\lik{\mathcal{L}}

\global\long\def\logll{\mathit{l}}

\global\long\def\vectorize#1{\mathbf{#1}}

\global\long\def\vt#1{\mathbf{#1}}

\global\long\def\gvt#1{\boldsymbol{#1}}

\global\long\def\idp{\ \bot\negthickspace\negthickspace\bot\ }
 \global\long\def\cdp{\idp}

\global\long\def\das{\triangleq}

\global\long\def\id{\mathbb{I}}

\global\long\def\idarg#1#2{\id\left\{  #1,#2\right\}  }

\global\long\def\iid{\stackrel{\text{iid}}{\sim}}

\global\long\def\bzero{\vt 0}

\global\long\def\bone{\mathbf{1}}

\global\long\def\boldm{\boldsymbol{m}}

\global\long\def\bff{\vt f}

\global\long\def\bx{\boldsymbol{x}}

\global\long\def\bl{\boldsymbol{l}}

\global\long\def\bu{\boldsymbol{u}}

\global\long\def\bo{\boldsymbol{o}}

\global\long\def\bh{\boldsymbol{h}}

\global\long\def\bs{\boldsymbol{s}}

\global\long\def\bz{\boldsymbol{z}}

\global\long\def\xnew{y}

\global\long\def\bxnew{\boldsymbol{y}}

\global\long\def\bX{\boldsymbol{X}}

\global\long\def\tbx{\tilde{\bx}}

\global\long\def\by{\boldsymbol{y}}

\global\long\def\bY{\boldsymbol{Y}}

\global\long\def\bZ{\boldsymbol{Z}}

\global\long\def\bU{\boldsymbol{U}}

\global\long\def\bv{\boldsymbol{v}}

\global\long\def\bn{\boldsymbol{n}}

\global\long\def\bV{\boldsymbol{V}}

\global\long\def\bI{\boldsymbol{I}}

\global\long\def\bw{\vt w}

\global\long\def\balpha{\gvt{\alpha}}

\global\long\def\bbeta{\gvt{\beta}}

\global\long\def\bmu{\gvt{\mu}}

\global\long\def\btheta{\boldsymbol{\theta}}

\global\long\def\blambda{\boldsymbol{\lambda}}

\global\long\def\bgamma{\boldsymbol{\gamma}}

\global\long\def\bpsi{\boldsymbol{\psi}}

\global\long\def\bphi{\boldsymbol{\phi}}

\global\long\def\bpi{\boldsymbol{\pi}}

\global\long\def\bomega{\boldsymbol{\omega}}

\global\long\def\bepsilon{\boldsymbol{\epsilon}}

\global\long\def\btau{\boldsymbol{\tau}}

\global\long\def\bvarphi{\boldsymbol{\varphi}}

\global\long\def\realset{\mathbb{R}}

\global\long\def\realn{\realset^{n}}

\global\long\def\integerset{\mathbb{Z}}

\global\long\def\natset{\integerset}

\global\long\def\integer{\integerset}

\global\long\def\natn{\natset^{n}}

\global\long\def\rational{\mathbb{Q}}

\global\long\def\rationaln{\rational^{n}}

\global\long\def\complexset{\mathbb{C}}

\global\long\def\comp{\complexset}

\global\long\def\compl#1{#1^{\text{c}}}

\global\long\def\and{\cap}

\global\long\def\compn{\comp^{n}}

\global\long\def\comb#1#2{\left({#1\atop #2}\right) }

\global\long\def\nchoosek#1#2{\left({#1\atop #2}\right)}

\global\long\def\param{\vt w}

\global\long\def\Param{\Theta}

\global\long\def\meanparam{\gvt{\mu}}

\global\long\def\Meanparam{\mathcal{M}}

\global\long\def\meanmap{\mathbf{m}}

\global\long\def\logpart{A}

\global\long\def\simplex{\Delta}

\global\long\def\simplexn{\simplex^{n}}

\global\long\def\dirproc{\text{DP}}

\global\long\def\ggproc{\text{GG}}

\global\long\def\DP{\text{DP}}

\global\long\def\ndp{\text{nDP}}

\global\long\def\hdp{\text{HDP}}

\global\long\def\gempdf{\text{GEM}}

\global\long\def\rfs{\text{RFS}}

\global\long\def\bernrfs{\text{BernoulliRFS}}

\global\long\def\poissrfs{\text{PoissonRFS}}

\global\long\def\grad{\gradient}
 \global\long\def\gradient{\nabla}

\global\long\def\partdev#1#2{\partialdev{#1}{#2}}
 \global\long\def\partialdev#1#2{\frac{\partial#1}{\partial#2}}

\global\long\def\partddev#1#2{\partialdevdev{#1}{#2}}
 \global\long\def\partialdevdev#1#2{\frac{\partial^{2}#1}{\partial#2\partial#2^{\top}}}

\global\long\def\closure{\text{cl}}

\global\long\def\cpr#1#2{\Pr\left(#1\ |\ #2\right)}

\global\long\def\var{\text{Var}}

\global\long\def\Var#1{\text{Var}\left[#1\right]}

\global\long\def\cov{\text{Cov}}

\global\long\def\Cov#1{\cov\left[ #1 \right]}

\global\long\def\COV#1#2{\underset{#2}{\cov}\left[ #1 \right]}

\global\long\def\corr{\text{Corr}}

\global\long\def\sst{\text{T}}

\global\long\def\SST{\sst}

\global\long\def\ess{\mathbb{E}}

\global\long\def\Ess#1{\ess\left[#1\right]}

\newcommandx\ESS[2][usedefault, addprefix=\global, 1=]{\underset{#2}{\ess}\left[#1\right]}

\global\long\def\fisher{\mathcal{F}}

\global\long\def\bfield{\mathcal{B}}
 \global\long\def\borel{\mathcal{B}}

\global\long\def\bernpdf{\text{Bernoulli}}

\global\long\def\betapdf{\text{Beta}}

\global\long\def\dirpdf{\text{Dir}}

\global\long\def\gammapdf{\text{Gamma}}

\global\long\def\gaussden#1#2{\text{Normal}\left(#1, #2 \right) }

\global\long\def\gauss{\mathbf{N}}

\global\long\def\gausspdf#1#2#3{\text{Normal}\left( #1 \lcabra{#2, #3}\right) }

\global\long\def\multpdf{\text{Mult}}

\global\long\def\poiss{\text{Pois}}

\global\long\def\poissonpdf{\text{Poisson}}

\global\long\def\pgpdf{\text{PG}}

\global\long\def\wshpdf{\text{Wish}}

\global\long\def\iwshpdf{\text{InvWish}}

\global\long\def\nwpdf{\text{NW}}

\global\long\def\niwpdf{\text{NIW}}

\global\long\def\studentpdf{\text{Student}}

\global\long\def\unipdf{\text{Uni}}

\global\long\def\transp#1{\transpose{#1}}
 \global\long\def\transpose#1{#1^{\mathsf{T}}}

\global\long\def\mgt{\succ}

\global\long\def\mge{\succeq}

\global\long\def\idenmat{\mathbf{I}}

\global\long\def\trace{\mathrm{tr}}

\global\long\def\argmax#1{\underset{_{#1}}{\text{argmax}} }

\global\long\def\argmin#1{\underset{_{#1}}{\text{argmin}\ } }

\global\long\def\diag{\text{diag}}

\global\long\def\norm{}

\global\long\def\spn{\text{span}}

\global\long\def\vtspace{\mathcal{V}}

\global\long\def\field{\mathcal{F}}
 \global\long\def\ffield{\mathcal{F}}

\global\long\def\inner#1#2{\left\langle #1,#2\right\rangle }
 \global\long\def\iprod#1#2{\inner{#1}{#2}}

\global\long\def\dprod#1#2{#1 \cdot#2}

\global\long\def\norm#1{\left\Vert #1\right\Vert }

\global\long\def\entro{\mathbb{H}}

\global\long\def\entropy{\mathbb{H}}

\global\long\def\Entro#1{\entro\left[#1\right]}

\global\long\def\Entropy#1{\Entro{#1}}

\global\long\def\mutinfo{\mathbb{I}}

\global\long\def\relH{\mathit{D}}

\global\long\def\reldiv#1#2{\relH\left(#1||#2\right)}

\global\long\def\KL{KL}

\global\long\def\KLdiv#1#2{\KL\left(#1\parallel#2\right)}
 \global\long\def\KLdivergence#1#2{\KL\left(#1\ \parallel\ #2\right)}

\global\long\def\crossH{\mathcal{C}}
 \global\long\def\crossentropy{\mathcal{C}}

\global\long\def\crossHxy#1#2{\crossentropy\left(#1\parallel#2\right)}

\global\long\def\breg{\text{BD}}

\global\long\def\lcabra#1{\left|#1\right.}

\global\long\def\lbra#1{\lcabra{#1}}

\global\long\def\rcabra#1{\left.#1\right|}

\global\long\def\rbra#1{\rcabra{#1}}

\global\long\def\modelname{\text{MC}\raisebox{3pt}{\tiny2}}

\global\long\def\DPM{\text{DPM}}

\global\long\def\nDP{\text{nDP}}

\global\long\def\nDPM{\text{nDPM}}

\section{Introduction\label{sec:intro}}

In many situations, content data naturally present themselves in groups,
e.g., students are grouped into classes, classes grouped into schools,
words grouped into documents, etc. Furthermore, each content group
can be associated with additional context information (teachers of
the class, authors of the document, time and location stamps). Dealing
with grouped data, a setting known as \emph{multilevel analysis} \cite{hox2010multilevel,diez2000multilevel},
has diverse application domains ranging from document modeling \cite{Blei_etal_03}
to public health \cite{leyland2001multilevel}. 

This paper considers specifically the multilevel clustering problem
in multilevel analysis: to jointly cluster both the content data and
their groups when there is group-level context information. By \emph{context,}
we mean a secondary data source attached to the group of primary \emph{content}
data. An example is the problem of clustering documents, where each
document is a group of words associated with group-level context information
such as time-stamps, list of authors, etc. Another example is image
clustering where visual image features (e.g. SIFT) are the\emph{ }content
and image tags are the context.

To cluster groups together, it is often necessary to perform dimensionality
reduction of the content data by forming content topics, effectively
performing clustering of the content as well. For example, in document
clustering, using bag-of-words directly as features is often problematic
due to the large vocabulary size and the sparsity of the in-document
word occurrences. Thus, a typical approach is to first apply dimensionality
reduction techniques such as LDA \cite{Blei_etal_03} or HDP \cite{Teh_etal_06hierarchical}
to find word topics (i.e., distributions on words), then perform document
clustering using the word topics and the document-level context information
as features. In such a cascaded approach, the dimensionality reduction
step (e.g., topic modeling) is not able to utilize the context information.
This limitation suggests that a better alternative is to perform context-aware
document clustering and topic modeling jointly. With a joint model,
one can expect to obtain improved document clusters as well as context-guided
content topics that are more predictive of the data.

Recent work has attempted to jointly capture word topics and document
clusters. Parametric approaches \cite{xieintegrating} are extensions
of the LDA \cite{Blei_etal_03} and require specifying the number
of topics and clusters in advance. Bayesian nonparametric approaches
including the nested Dirichlet process (nDP) \cite{Rodriguez_etal_08nested}
and the multi-level clustering hierarchical Dirichlet Process (MLC-HDP)
\cite{wulsin_2012_hierarchical} can automatically adjust the number
of clusters. We note that none of these methods can utilize context
data.

This paper propose the \emph{Multilevel Clustering with Context} ($\modelname$),
a Bayesian nonparametric model to jointly cluster both content and
groups while fully utilizing group-level context. Using the Dirichlet
process as the building block, our model constructs a product base-measure
with a nested structure to accommodate both content and context observations.
The $\modelname$ model possesses properties that link the nested
Dirichlet process (nDP) and the Dirichlet process mixture model (DPM)
in an interesting way: integrating out all contents results in the
DPM over contexts, whereas integrating out group-level context results
in the nDP mixture over content variables. For inference, we provide
an efficient collapsed Gibbs sampling procedure for the model.

The advantages of our model are: (1) the model automatically discovers
the (unspecified) number of groups clusters and the number of topics
while fully utilizing the context information; (2) content topic modeling
is informed by group-level context information, leading to more predictive
content topics; (3) the model is robust to partially missing context
information. In our experiments, we demonstrate that our proposed
model achieves better document clustering performances and more predictive
word topics in real-world datasets in both text and image domains.

\section{Related Background\label{sec:relatedworks}}

There have been extensive works on clustering documents in the literature.
Due to limited scope of the paper, we only describe works closely
related to probabilistic topic models. We note that standard topic
models such as LDA \cite{Blei_etal_03} or its nonparametric Bayesian
counter part, HDP \cite{Teh_etal_06hierarchical} exploits the group
structure for word clustering. However these models do not cluster
documents.

An approach to document clustering is to employ a two-stage process.
First, topic models (e.g. LDA or HDP) are applied to extract the topics
and their mixture proportion for each document. Then, this is used
as feature input to another clustering algorithm. Some examples of
this approach include the use of LDA+Kmeans for image clustering \cite{xuan2011finding,elango2005clustering}
and HDP+Affinity Propagation for clustering human activities \cite{nguyen_phung_gupta_venkatesh_percom13}.

A more elegant approach is to simultaneously cluster documents and
discover topics. The first Bayesian nonparametric model proposed for
this task is the nested Dirichlet Process (nDP) \cite{Rodriguez_etal_08nested}
where documents in a cluster share the same distribution over topic
atoms. Although the original nDP does not force the topic atoms to
be shared across document clusters, this can be achieved by simply
introducing a DP prior for the nDP base measure. The same observation
was also made by \cite{wulsin_2012_hierarchical} who introduced the
MLC-HDP, a 3-level extension to the nDP. This model thus can cluster
words, documents and document-corpora with shared topic atoms throughout
the group hierarchy. Xie et al \cite{xieintegrating} recently introduced
the Multi-Grain Clustering Topic Model which allows mixing between
global topics and document-cluster topics. However, this is a parametric
model which requires fixing the number of topics in advance. More
crucially, all of these existing models do not attempt to utilize
group-level context information.

\subsection*{Modelling with Dirichlet Process}

We provide a brief account of the Dirichlet process and its variants.
The literature on DP is vast and we refer to \cite{Hjort_etal_bk10_bayesian}
for a comprehensive account. Here we focus on DPM, HDP and nDP which
are related to our work.

\emph{Dirichlet process} \cite{Ferguson_73bayesian} is a basic building
block in Bayesian nonparametrics. Let $\left(\Theta,\borel,H\right)$
be a probability measure space, and $\gamma$ is a positive number,
a Dirichlet process $\dirproc\left(\gamma,H\right)$ is a distribution
over discrete random probability measure $G$ on $\left(\Theta,\borel\right)$.
Sethuraman \cite{Sethuraman_94constructive} provides an alternative
constructive definition which makes the discreteness property of a
draw from a Dirichlet process explicit via the stick-breaking representation:
$G=\sum_{k=1}^{\infty}\beta_{k}\delta_{\phi_{k}}$ where $\phi_{k}\iid H,k=1,\ldots,\infty$
and $\bbeta=\left(\beta_{k}\right)_{k=1}^{\infty}$ are the weights
constructed through a `stick-breaking' process $\beta_{k}=v_{k}\prod_{s<k}\left(1-v_{s}\right)$
with $v_{k}\iid\betapdf\left(1,\gamma\right)$. It can be shown that
$\sum_{k=1}^{\infty}\beta_{k}=1$ with probability one, and as a convention
\cite{Pitman_02poisson}, we hereafter write $\bbeta\sim\gempdf\left(\gamma\right)$. 

Due to its discrete nature, Dirichlet process has been widely used
in Bayesian mixture models as the prior distribution on the mixing
measures, each is associated with an atom $\phi_{k}$ in the stick-breaking
representation of $G$ above. A likelihood kernel $F\left(\cdot\right)$
is used to generate data  $x_{i}\gv\phi_{k}\iid F\left(\cdot\gv\phi_{k}\right)$,
resulting in a model known as the \emph{Dirichlet process mixture
model} (DPM), pioneered by the work of \cite{Antoniak_74mixtures}
and subsequently developed by many others. In section \ref{sec:model}
we provide a precise definition for DPM.

While DPM models exchangeable data within a \emph{single} group, the
Dirichlet process can also be constructed hierarchically to provide
prior distributions over \emph{multiple} exchangeable groups. Under
this setting, each group is modelled as a DPM and these models are
`linked' together to reflect the dependency among them --  a formalism
which is generally known as dependent Dirichlet processes \cite{MacEachern_99dependent}.
One particular attractive approach is the \emph{hierarchical Dirichlet
processes} \cite{Teh_etal_06hierarchical} which posits the dependency
among the group-level DPM by another Dirichlet process, i.e., $G_{j}\mid\alpha,G_{0}\sim\dirproc\left(\alpha,G_{0}\right)$
and $G_{0}\gv\gamma,H\sim\dirproc\left(\gamma,H\right)$ where $G_{j}$
is the prior for the $j$-th group, linked together via a discrete
measure $G_{0}$ whose distribution is another DP.

Yet another way of using DP to model multiple groups is to construct
random measure in a nested structure in which the DP base measure
is itself another DP. This formalism is the \emph{nested Dirichlet
Process} \cite{Rodriguez_etal_08nested}, specifically $G_{j}\iid U$
where $U\sim\DP\left(\alpha\times\DP\left(\gamma H\right)\right)$.
Modeling $G_{j}$ (s) hierarchically as in HDP and nestedly as in
nDP yields different effects. HDP focuses on exploiting statistical
strength across groups via sharing atoms $\phi_{k}$ (s), but it does
not partition groups into clusters. This statement is made precisely
by noting that $P\left(G_{j}=G_{j'}\right)=0$ in HDP. Whereas, nDP
emphasizes on inducing clusters on both observations and distributions,
hence it partitions groups into clusters. To be precise, the prior
probability of two groups being clustered together is $P\left(G_{j}=G_{j'}\right)=\frac{1}{a+1}$.
Finally we note that this original definition of nDP in \cite{Rodriguez_etal_08nested}
does not force the atoms to be shared across clusters of groups, but
this can be achieved by simply introducing a DP prior for the nDP
base measure, a modification that we use in this paper. This is made
clearly in our definition for nDP mixture in section \ref{sec:model}.

\section{Multilevel Clustering with Contexts\label{sec:model}}

\subsection{Model description and stick-breaking}

Consider data presented in a two-level group structure as follows.
Denote by $J$ the number of groups; each group $j$ contains $N_{j}$
exchangeable data points, represented by $\bw_{j}=\left\{ w_{j1},w_{j2},\ldots,w_{jN_{j}}\right\} $.
For each group $j$, the group-specific context data is denoted by
$x_{j}$. Assuming that the groups are exchangeable, the overall data
is $\left\{ (x_{j},\bw_{j})\right\} _{j=1}^{J}$. The collection $\left\{ \bw_{1},\ldots,\bw_{J}\right\} $
represents observations of the group contents, and $\left\{ x_{1},\ldots,x_{J}\right\} $
represents observations of the group-level contexts.

We now describe the generative process of \modelname ~that generates
a two-level clustering of this data. We use a group-level DP mixture
to generate an infinite cluster model for groups. Each group cluster
$k$ is associated with an atom having the form of a pair $(\phi_{k},Q_{k}^{*})$
where $\phi_{k}$ is a parameter that generates the group-level contexts
within the cluster and $Q_{k}^{*}$ is a measure that generates the
group contents within the same cluster. 

To generate atomic pairs of context parameter and measure-valued content
parameter, we introduce a product base-measure of the form $H\times\DP(vQ_{0})$
for the group-level DP mixture. Drawing from a DP mixture with this
base measure, each realization is a pair $(\theta_{j},Q_{j})$; $\theta_{j}$
is then used to generate the context $x_{j}$ and $Q_{j}$ is used
to repeatedly produce the set of content observations $w_{ji}$ within
the group $j$. Specifically,
\begin{align}
U\sim\dirproc\left(\alpha(H\times\DP(vQ_{0}))\right) & \,\text{where }Q_{0}\sim\dirproc\left(\eta S\right)\nonumber \\
(\theta_{j},Q_{j})\iid U & \ \text{for each group\ }j\label{eq:MC2}\\
x_{j}\sim F(.\vert\theta_{j}),\quad\varphi_{ji} & \iid Q_{j},\quad w_{ji}\sim Y\left(.\vert\varphi_{ji}\right)\nonumber 
\end{align}
In the above, $H$ and $S$ are respectively base measures for context
and content parameters $\theta_{j}$ and $\varphi_{ji}$. The context
and content observations are then generated via the likelihood kernels
$F\left(\cdot\gv\theta_{j}\right)$ and $Y\left(\cdot\gv\varphi_{ji}\right)$.
To simplify inference, $H$ and $S$ are assumed to be conjugate to
$F$ and $Y$ respectively. The generative process is illustrated
in Figure \ref{fig:J3}.

\begin{figure}
\begin{centering}
\includegraphics[scale=0.25]{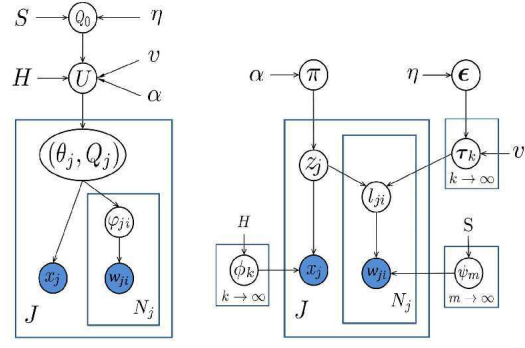}\vspace{-0.1in}

\par\end{centering}

\caption{Graphical model representation for the proposed model. Right figure
illustrates a stick breaking representation.\label{fig:J3}}
\end{figure}

\subsubsection*{Stick-breaking representation}

We now derive the stick-breaking construction for \modelname$\ $
$\ $where all the random discrete measures are specified by a distribution
over integers and a countable set of atoms. The random measure $U$
in Eq. (\ref{eq:MC2}) has the stick-breaking form:\vspace{-0.05in}
\begin{align}
U & =\sum_{k=1}^{\infty}\pi_{k}\delta_{\left(\phi_{k},Q_{k}^{*}\right)}\label{eq:}
\end{align}
where $\pi\sim\gempdf\left(\alpha\right)$ and $\left(\phi_{k},Q_{k}^{*}\right)\iid H\times\dirproc\left(vQ_{0}\right)$.
Equivalently, this means $\phi_{k}$ is drawn i.i.d. from $H$ and
$Q_{k}^{*}$ drawn i.i.d. from $\dirproc\left(vQ_{0}\right)$. Since
$Q_{0}\sim\dirproc\left(\eta S\right)$, $Q_{0}$ and $Q_{k}^{*}$
have the standard HDP \cite{Teh_etal_06hierarchical} stick-breaking
forms: $Q_{0}=\sum_{m=1}^{\infty}\epsilon_{m}\delta_{\psi_{m}}$where
$\bepsilon\sim\textrm{GEM}(\eta)$, $\psi_{m}\iid S$; $Q_{k}^{*}=\sum_{m=1}^{\infty}\tau_{k,m}\delta_{\psi_{m}}$
where $\btau_{k}=\left(\tau_{k1},\tau_{k2},\ldots\right)\sim\dirproc\left(v,\bepsilon\right)$.

For each group $j$ we sample the parameter pair $\left(\theta_{j},Q_{j}\right)\iid U$;
equivalently, this means drawing $z_{j}\iid\pi$ and letting $\theta_{j}=\phi_{z_{j}}$
and $Q_{j}=Q_{z_{j}}^{*}$. For the $i$-th content data within the
group $j$, the content parameter $\varphi_{ji}$ is drawn $\iid Q_{j}=Q_{z_{j}}^{*}$;
equivalently, this means drawing $l_{ji}\iid\tau_{z_{j}}$ and letting
$\varphi_{ji}=\psi_{l_{ji}}$. Figure \ref{fig:J3} presents the graphical
model of this stick-breaking representation.

\subsection{Inference and Polya Urn View}

We use collapsed Gibbs sampling, integrating out $\phi_{k}$(s), $\psi_{m}$(s),
$\pi$ and $\tau_{k}$ (s). Latent variables $\bz$, $\bl$, $\bepsilon$
and the hyper-parameters $\alpha$, $v$, $\eta$ will be resampled.
We only describe the key inference steps in sampling $\bz,\bl$ and
$\bepsilon$ here and refer to Appendix \ref{sub:Model-Inference-Derivations}
for the rest of the details (including how to sample the hyper-parameters).

\textbf{Sampling} $\bz$. The required conditional distribution is
$p(z_{j}=k\gv\bz_{-j},\bl,\bx,\alpha,H)\propto$
\begin{align*}
 & p\left(z_{j}=k|\bz_{-j},\alpha\right)p\left(x_{j}|z_{j}=k,\bz_{-j},\bx_{-j},H\right)\\
 & \times p\left(l_{j*}|z_{j}=k,\bl_{-j*},\bz_{-j},\bepsilon,v\right)
\end{align*}
The first term can be recognized as a form of the Chinese restaurant
process (CRP). The second term is the predictive likelihood for the
context observations under the component $\phi_{k}$ after integrating
out $\phi_{k}$. This can be evaluated analytically due to conjugacy
of $F$ and $H$. The last term is the predictive likelihood for the
group content-index $l_{j*}=\left\{ l_{ji}\vert i=1\ldots N_{j}\right\} $.
Since $l_{ji}\gv z_{j}=k\iid\multpdf\left(\btau_{k}\right)$ where
$\btau_{k}\sim\dirpdf\left(v\epsilon_{1},\ldots,v\epsilon_{M},\epsilon_{\text{new}}\right)$,
the last term can also be evaluated analytically by integrating out
$\btau_{k}$ using the Multinomial-Dirichlet conjugacy property.

\textbf{Sampling} $\bl$. Let $w_{-ji}$ be the same set as $\bw$
excluding $w_{ji}$, let $w_{-ji}(m)=\left\{ w_{j'i'}\vert(j'i')\neq(ji)\wedge l_{j'i'}=m\right\} $
and $\bl_{-ij}(k)=\left\{ l_{j'i'}\vert(j'i')\neq(ji)\wedge z_{j'}=k\right\} $.
Then $p\left(l_{ji}=m\mid l_{-ji},z_{j}=k,z_{-j},v,\bw,\bepsilon,S\right)\propto$
\begin{align*}
 & p(w_{ji}\vert\bl,w_{-ji},S)\ p(l_{ji}=m\vert\bl_{-ji},z_{j}=k,z_{-j},\bepsilon,v)\\
= & p\left(w_{ji}\mid w_{-ji}(m),S\right)p\left(l_{ji}=m\mid\bl_{-ji}(k),\bepsilon,v\right)
\end{align*}

The first term is the predictive likelihood under mixture component
$\psi_{m}$ after integrating out $\psi_{m}$, which can be evaluated
analytically due to the conjugacy of $Y$ and $S$. The second term
is in the form of a CRP similar to the one that arises during inference
for HDP \cite{Teh_etal_06hierarchical}. 

\textbf{Sampling} $\bepsilon$. Sampling $\bepsilon$ requires information
from both $\bz$ and $\bl$.
\begin{align}
p\left(\epsilon\mid\bl,\bz,v,\eta\right) & \propto p\left(\bl\mid\bepsilon,v,\bz,\eta\right)\times p\left(\bepsilon\mid\eta\right)\label{eq:sampling epsilon-1-1}
\end{align}
Using a similar strategy in HDP, we introduce auxiliary variables
$\left(o_{km}\right)$, then alternatively sample together with $\bepsilon$:
\begin{align*}
p\left(o_{km}=h\gv\cdot\right) & \propto\textrm{Stirl}\left(h,n_{km}\right)(v\epsilon_{m})^{h},\, h=0,1,\ldots,n_{km}\\
p\left(\bepsilon\mid\cdot\right) & \propto\epsilon_{\textrm{new}}^{\eta-1}\prod_{m=1}^{M}\epsilon_{m}^{\sum_{k}o_{km}-1}
\end{align*}
where $\textrm{Stirl}\left(h,n_{km}\right)$ is the Stirling number
of the first kind, $n_{km}$ is the count of seeing the pair $\left(z_{j}=k,l_{ji}=m\right):\forall i,j$,
and finally $M$ is the current number of active content topics. It
clear that $o_{km}$ can be sampled from a Multinomial distribution
and $\epsilon$ from an $\left(M+1\right)$-dim Dirichlet distribution.

\subsubsection*{Polya Urn View}

Our model exhibits a Polya-urn view using the analogy of a fleet of
buses, driving customers to restaurants. Each bus represents a group
and customers on the bus are data points within the group. For each
bus $j$, $z_{j}$ acts as the index to the restaurant for its destination.
Thus, buses form clusters at their destination restaurants according
to a CRP: a new bus drives to an existing restaurant with the probability
proportional to the number of other buses that have arrived at that
restaurant, and with probability proportional to $\alpha$, it goes
to a completely new restaurant. 

Once all the buses have delivered customers to the restaurants, \emph{all
customers at the restaurants start to behave in the same manner as
in a Chinese restaurant franchise (CRF) process}: customers are assigned
tables according to a restaurant-specific CRP; tables are assigned
with dishes $\psi_{m}$ (representing the content topic atoms) according
to a global franchise CRP. In addition to the usual CRF, at restaurant
$k$, a single dessert $\phi_{k}$ (which represents the context-generating
atom, drawing $\iid$ from $H$) will be served to all the customers
at that restaurant. Thus, every customer\emph{ }on the same bus\emph{
$j$ }will be served the same\emph{ }dessert \emph{$\phi_{z_{j}}$}.
We observe three sub-CRPs, corresponding to the three DP(s) in our
model: the CRP at the dish level is due to the $\dirproc\left(\eta S\right)$,
the CRP forming tables inside each restaurant is due to the $\DP(vQ_{0})$,
and the CRP aggregating buses to restaurants is due to the $\dirproc\left(\alpha(H\times\DP(vQ_{0}))\right)$.

\subsection{Marginalization property}

We study marginalization property for our model when either the content
topics $\varphi_{ji}$ (s) or context topics $\theta_{j}$ (s) are
marginalized out. Our main result is established in Theorem \ref{proof_Theorem4}
where we show an interesting link to nested DP and DPM via our model.

Let $H$ be a measure over some measurable spaces $(\Theta,\Sigma)$.
Let $\mathbb{P}$ be the set of all measures over $(\Theta,\Sigma)$,
suitably endowed with some $\sigma$-algebra. Let $G\sim\DP(\alpha H)$
and $\theta_{i}\iid G$. The collection $(\theta_{i})$ then follows
the DP mixture distribution which is defined formally below.
\begin{defn}
(DPM)\label{def:DPM} A DPM is a probability measure over $\Theta^{n}\ni\left(\theta_{1},\ldots,\theta_{n}\right)$
with the usual product sigma algebra $\Sigma^{n}$ such that for every
collection of measurable sets $\left\{ \left(S_{1},\ldots,S_{n}\right):S_{i}\in\Sigma,i=1,\ldots,n\right\} $:
\begin{align*}
 & \DPM(\theta_{1}\in S_{1},\ldots,\theta_{n}\in S_{n}\vert\alpha,H)\\
 & \qquad=\int\prod_{i=1}^{n}G\left(S_{i}\right)\DP\left(dG\gv\alpha H\right)
\end{align*}
\vspace{-0.2in}

\end{defn}
We now state a result regarding marginalization of draws from a DP
mixture with a joint base measure. Consider two measurable spaces
$(\Theta_{1},\Sigma_{1})$ and $(\Theta_{2},\Sigma_{2})$ and let
$(\Theta,\Sigma)$ be their product space where $\Theta=\Theta_{1}\times\Theta_{2}$
and $\Sigma=\Sigma_{1}\times\Sigma_{2}$. Let $H^{*}$ be a measure
over the product space $\Theta=\Theta_{1}\times\Theta_{2}$ and let
$H_{1}$ be the marginal of $H^{*}$ over $\Theta_{1}$ in the sense
that for any measurable set $A\in\Sigma_{1}$, $H_{1}\left(A\right)=H^{*}\left(A\times\Theta_{2}\right)$.
Then drawing $(\theta_{i}^{(1)},\theta_{i}^{(2)})$ from a DP mixture
with base measure $\alpha H$ and marginalizing out $(\theta_{i}^{(2)})$
is the same as drawing $(\theta_{i}^{(1)})$ from a DP mixture with
base measure $H_{1}$. Formally
\begin{prop}
\label{thm:marginalization_theorem}Denote by $\theta_{i}$ the pair
\textup{$\left(\theta_{i}^{\left(1\right)},\theta_{i}^{\left(2\right)}\right)$,
there holds} 
\begin{align*}
 & \DPM\left(\theta_{1}^{\left(1\right)}\in S_{1},\ldots,\theta_{n}^{\left(1\right)}\in S_{n}\gv\alpha H_{1}\right)\\
 & \quad=\DPM\left(\theta_{1}\in S_{1}\times\Theta_{2},\ldots,\theta_{n}\in S_{n}\times\Theta_{2}\gv\alpha H^{*}\right)
\end{align*}
for every collection of measurable sets $\left\{ \left(S_{1},\ldots,S_{n}\right):S_{i}\in\Sigma_{1},i=1,\ldots,n\right\} $. \end{prop}
\begin{proof}
see Appendix \ref{Proof_Prop2}.
\end{proof}
Next we give a formal definition for the nDP mixture: $\varphi_{ji}\iid Q_{j}$,
$Q_{j}\iid U$, $U\sim\DP(\alpha\DP(vQ_{0}))$, $Q_{0}\sim DP(\eta S)$.
\begin{defn}
\label{def:nDPM}(nested DP Mixture) An $\nDPM$ is a probability
measure over $\Theta^{\sum_{j=1}^{J}N_{j}}\ni\left(\varphi_{11},\ldots,\varphi_{1N_{1}},\ldots,\varphi_{JN_{J}}\right)$
equipped with the usual product sigma algebra $\Sigma^{N_{1}}\times\ldots\times\Sigma^{N_{J}}$
such that for every collection of measurable sets $\left\{ \left(S_{ji}\right):S_{ji}\in\Sigma,j=1,\ldots,J,\, i=1\ldots,N_{j}\right\} $:
\begin{align*}
 & \nDPM(\varphi_{ji}\in S_{ji},\forall i,j\vert\alpha,v,\eta,S)\\
 & \,\,\,\,=\int\int\left\{ \prod_{j=1}^{J}\int\prod_{i=1}^{N_{j}}Q_{j}\left(S_{ji}\right)U\left(dQ_{j}\right)\right\} \\
 & \,\,\,\,\,\,\,\,\,\times\DP\left(dU\gv\alpha\DP\left(vQ_{0}\right)\right)\DP\left(dQ_{0}\gv\eta,S\right)
\end{align*}

We now have the sufficient formalism to state the marginalization
result for our model.\end{defn}
\begin{thm}
\label{prop:main}Given $\alpha,H$ and $\alpha,v,\eta,S$, let $\btheta=\left(\theta_{j}:\forall j\right)$
and $\bvarphi=\left(\varphi_{ji}:\forall j,i\right)$ be generated
as in Eq (\ref{eq:MC2}). Then, marginalizing out $\bvarphi$ results
in $\DPM\left(\btheta\gv\alpha,H\right)$, whereas marginalizing out
$\btheta$ results in $\nDPM\left(\bvarphi\vert\alpha,v,\eta,S\right)$.\end{thm}
\begin{proof}
We sketch the main steps, Appendix \ref{proof_Theorem4} provides
more detail. Let $H^{*}=H_{1}\times H_{2}$, we note that when either
$H_{1}$ or $H_{2}$ are random, a result similar to Proposition \ref{Proof_Prop2}
still holds by taking the expectation on both sides of the equality.
Now let $H_{1}=H$ and $H_{2}=\DP\left(vQ_{0}\right)$ where $Q_{0}\sim\DP(\eta S)$
yields the proof for the marginalization of $\bvarphi$; let $H_{1}=\DP\left(vQ_{0}\right)$
and $H_{2}=H$ yields the proof for the marginalization of $\btheta$.
\end{proof}

\section{Experiments}

We first evaluate the model via simulation studies, then demonstrate
its applications on text and image modeling using three real-world
datasets. Throughout this section, unless explicitly stated, discrete
data is modeled by Multinomial with Dirichlet prior, while continuous
data is modeled by Gaussian (unknown mean and unknown variance) with
Gaussian-Gamma prior.

\subsection{Simulation studies}

\begin{figure*}[t]
\begin{centering}
\begin{tabular}{cc}
\includegraphics[width=0.95\columnwidth]{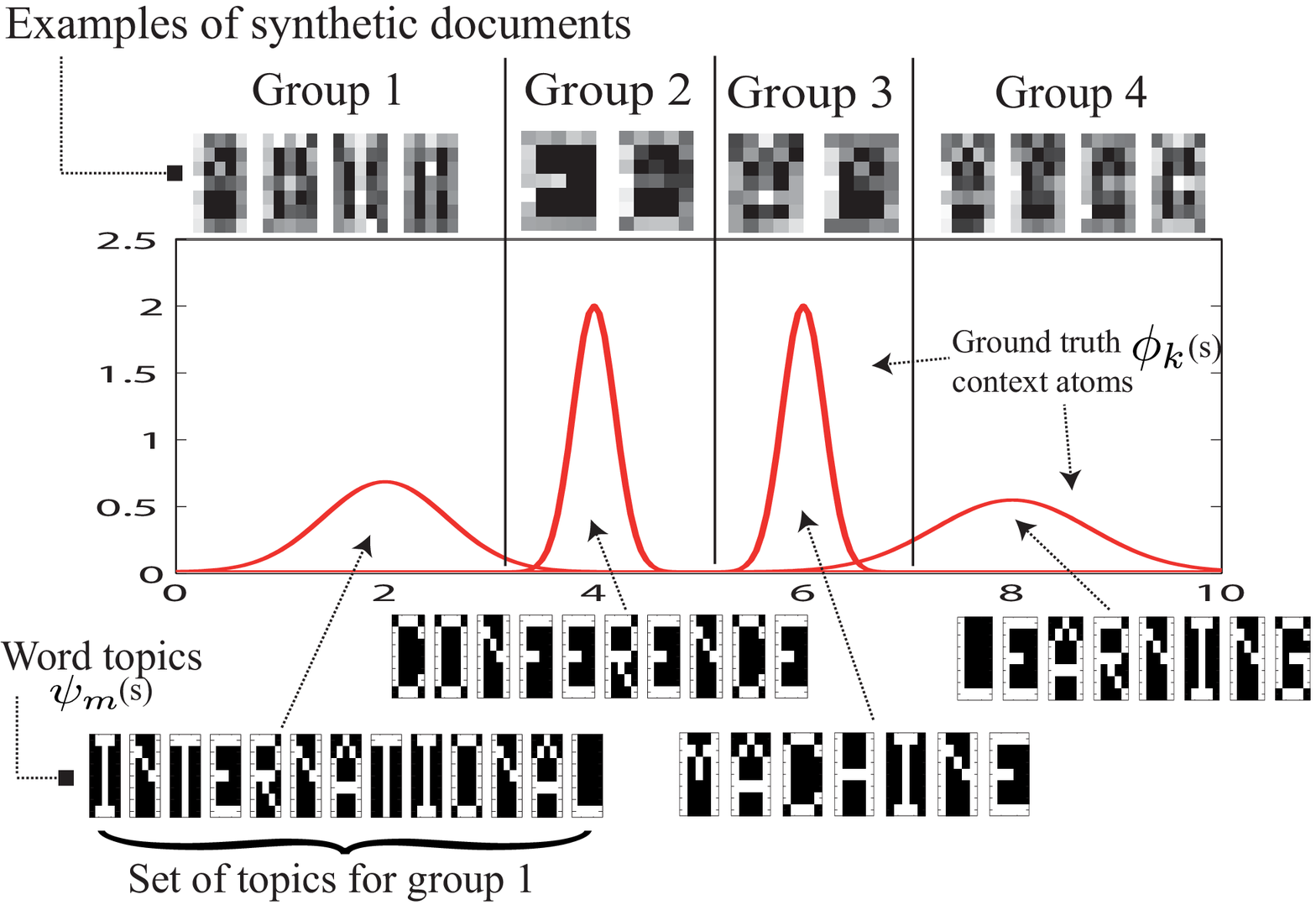} & \includegraphics[width=0.92\columnwidth]{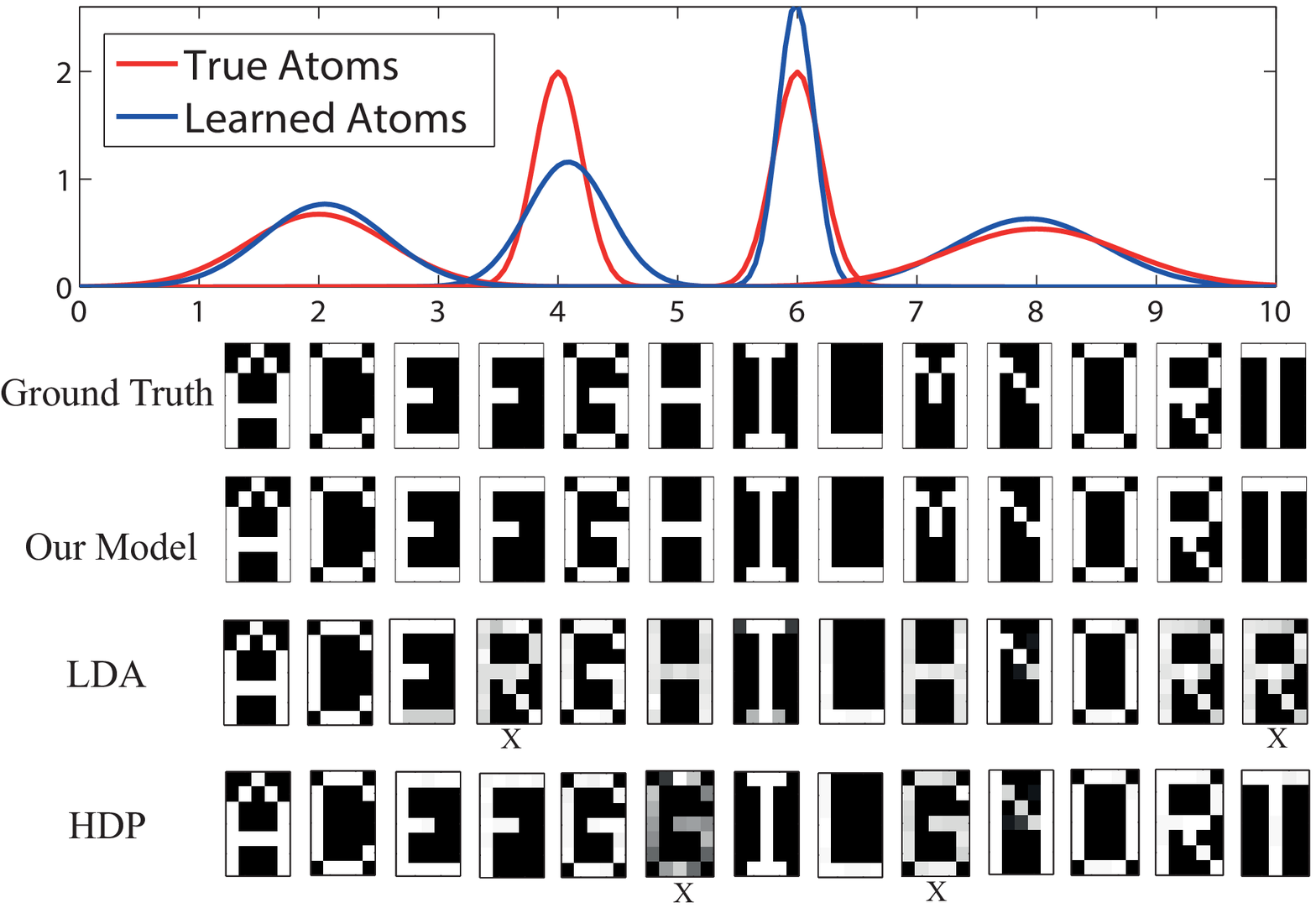}\tabularnewline
\end{tabular}\vspace{-0.1in}

\par\end{centering}

\caption{Results from simulation study. Left: illustration of data generation
with ground truth for context atoms are 4 univariate Gaussians centered
at $2,4,6$ and $8$ respectively (different variances). Right: Our
model recovers the correct 4 group clusters, their context distributions
and the set of shared topics. LDA and HDP are unable to recover the
true content topics without using contexts. \label{fig:Synthetic-Experiment.}}
\end{figure*}

The main goal is to investigate the posterior consistency of the model,
i.e., its ability to recover the true group clusters, context distribution
and content topics. To synthesize the data, we use $M=13$ topics
 which are the 13 unique letters in the ICML string ``INTERNATIONAL
CONFERENCE MACHINE LEARNING''. Similar to \cite{griffiths2004finding},
each topic $\psi_{m}$ is a distribution over 35 words (pixels) and
visualized as a $7\times5$ binary image. We generate $K=4$ clusters
of 100 documents each. For each cluster, we choose a set of topics
corresponding to letters in the each of 4 words in the ICML string.
The topic mixing distribution $\tau_{k}$ is an uniform distribution
over the chosen topic letters. Each cluster is also assigned a context-generating
univariate Gaussian distribution. These generating parameters are
shown in Figure \ref{fig:Synthetic-Experiment.} (left). Altogether
we have $J=400$ documents; for each document we sample $N_{j}=50$
words and a context variable $x_{j}$ drawing from the cluster-specific
Gaussian.

We model the word $w_{ji}$ with Multinomial and Gaussian for context
$x_{j}$. After 100 Gibbs iterations, the number of context and content
topics ($K=4,M=13$) are recovered correctly: the learned context
atoms $\phi_{k}$ and topic $\psi_{m}$ are almost identical to the
ground truth (Figure \ref{fig:Synthetic-Experiment.}, right) and
the model successfully identifies the $4$ clusters of documents with
topics corresponding to the 4 words in the ICML string.

To demonstrate the importance of context observation, we then run
LDA and HDP with only the word observations (ignoring context) where
the number of topic of LDA is set to 13. As can be seen from Figure
\ref{fig:Synthetic-Experiment.} (right), LDA and HDP have problems
in recovering the true topics. They cannot distinguish small differences
between the overlapping character topics (e.g M vs N, or I vs T).
Further analysis of the role of context in $\modelname$ is provided
in Appendix \ref{sub:Relative-Roles-of}.

\subsection{{\normalsize Experiments with Real-World Datasets}}

\begin{table*}[t]
\begin{centering}
\begin{tabular}{|c|c|c|c|c|c|}
\hline 
\multirow{2}{*}{Method} & \multicolumn{4}{c|}{Perplexity (\textbf{\emph{on words only}})} & \multirow{2}{*}{Feature used}\tabularnewline
\cline{2-5} 
 & PNAS & (K,M) & NIPS & (K,M) & \tabularnewline
\hline 
HDP \cite{Teh_etal_06hierarchical} & $3027.5$ & $(-,86)$ & $1922.1$ & $(-,108)$ & words\tabularnewline
\hline 
npTOT \cite{dubey2012non,phung_nguyen_bui_nguyen_venkatesh_tr12} & $2491.5$ & $(-,145)$ & $1855.33$ & $(-,94)$ & words+timestamp\tabularnewline
\hline 
$\modelname$ without context  & $1742.6$ & $(40,126)$ & $1583.2$ & $(19,61)$ & words\tabularnewline
\hline 
$\modelname$ with titles & -- & -- & $1393.4$ & $(32,80)$ & words+title\tabularnewline
\hline 
$\modelname$ with authors & -- & -- & $1246.3$ & $(8,55)$ & words+authors\tabularnewline
\hline 
$\modelname$ with timestamp & $\mathbf{895.3}$ & $(12,117)$ & \textbf{$\mathbf{984.7}$} & $(15,95)$ & words+timestamp\tabularnewline
\hline 
\end{tabular}
\par\end{centering}

\centering{}\caption{Perplexity evaluation on PNAS and NIPS datasets. (K,M) is (\#cluster,\#topic).
(Note: missing results are due to title and author information not
available in PNAS dataset). \label{tab:Perplexity Table}}
\end{table*}

We use two standard NIPS and PNAS text datasets, and the NUS-WIDE
image dataset.

\emph{NIPS} contains 1,740 documents with vocabulary size 13,649 (excluding
stop words); timestamps (1987-1999), authors (2,037) and title information
are available and used as group-level context. \emph{PNAS} contains
79,800 documents, vocab size = 36,782 with publication timestamp (915-2005).
For \emph{NUS-WIDE} we use a subset of the 13-class animals %
\footnote{downloaded from http://www.ml-thu.net/\textasciitilde{}jun/data/%
} comprising of 3,411 images (2,054 images for training and 1357 images
for testing) with off-the-shelf features including 500-dim bag-of-word
SIFT vector and 1000-dim bag-of-tag annotation vector.

\paragraph*{\vspace{-0.4in}
}

\paragraph*{Text Modeling with Document-Level Contexts\protect \\
}

We use NIPS and PNAS datasets with 90\% for training and 10\% for
held-out perplexity evaluation. We compare the perplexity with HDP
\cite{Teh_etal_06hierarchical} where no group-level context can be
used, and npTOT \cite{dubey2012non} where only timestamp information
can be used. We note that unlike our model, npTOT requires replication
of document timestamp for \emph{every} word in the document, which
is somewhat unnatural. 

We use perplexity score \cite{Blei_etal_03} on held-out data as performance
metric, defined as%
\footnote{Appendix \ref{sub:Perplexity-Evaluation} provides further details
on how to derive this score from our model%
} $\exp\left\{ -\sum_{j=1}^{J}\log p\left(\bw_{j}^{\text{test}}\mid\bx^{\text{train}},\bw^{\text{train}}\right)/\left(\sum_{j}N_{j}^{\text{test}}\right)\right\} $.
To ensure fairness and comparable evaluation, \emph{only words} in
held-out data is used to compute the perplexity. We use univariate
Gaussian  for timestamp and Multinomial distributions for words,
tags and authors. We ran collapsed Gibbs for 500 iterations after
100 burn-in samples. 

Table \ref{tab:Perplexity Table} shows the results where $\modelname$
achieves significant better performance. This shows that group-level
context information during training provide useful guidance for the
modelling tasks. Regarding the informative aspect of group-level context,
we achieve better perplexity with timestamp information than with
titles and authors. This may be explained by the fact that 1361 authors
(among 2037) show up only once in the data while title provides little
additional information than what already in that abstracts. Interestingly,
without the group-level context information, our model still predicts
the held-out words better than HDP. This suggests that inducing partitions
over documents simultaneously with topic modelling is beneficial.

Beyond the capacity of HDP and npTOT, our model can induce clusters
over documents (value of $K$ in Table \ref{tab:Perplexity Table}).
Figure \ref{fig:Examples-of-authorNIPS} shows an example of one such
document cluster discovered from NIPS data with authors as context.
\begin{figure}[th]
\centering{}%
\begin{tabular}{|>{\raggedright}p{0.98\columnwidth}|}
\hline 
\includegraphics[width=0.8\columnwidth]{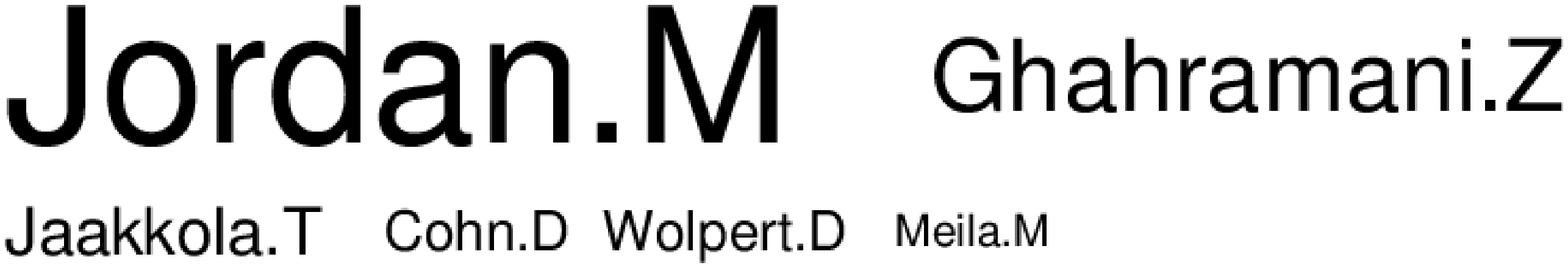}\tabularnewline
\hline 
{\tiny ~~~~~On the use of evidence in neural networks {[}1993{]}}{\tiny \par}

{\tiny ~~~~~Supervised Learning from Incomplete Data via an EM
{[}1994{]}}{\tiny \par}

{\tiny ~~~~~Fast Learning by Bounding Likelihoods in ... Networks
{[}1996{]}}{\tiny \par}

{\tiny ~~~~~Factorial Hidden Markov Models {[}1997{]}}{\tiny \par}

{\tiny ~~~~~Estimating Dependency Structure as a Hidden Variable
{[}1998{]}}{\tiny \par}

{\tiny ~~~~~Maximum Entropy Discrimination {[}1999{]}}\tabularnewline
\hline 
~~~\includegraphics[width=0.8\columnwidth]{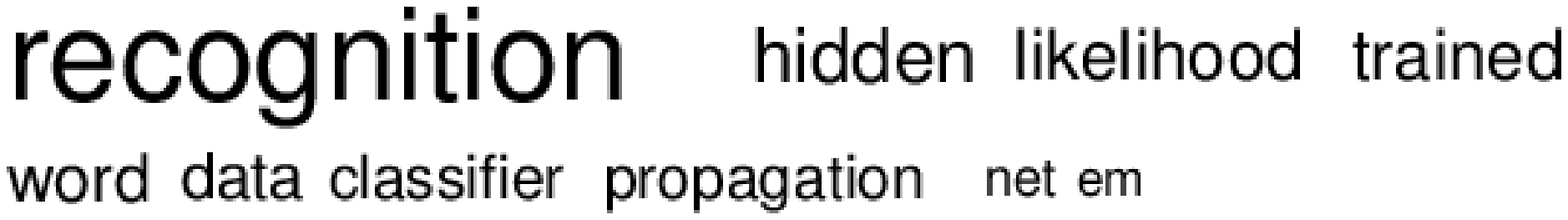}

~~~\includegraphics[width=0.8\columnwidth]{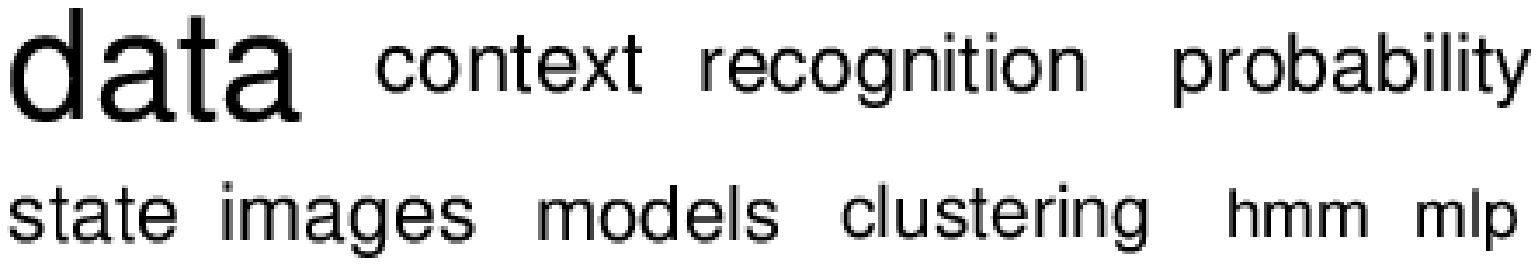}

~~~\includegraphics[width=0.8\columnwidth]{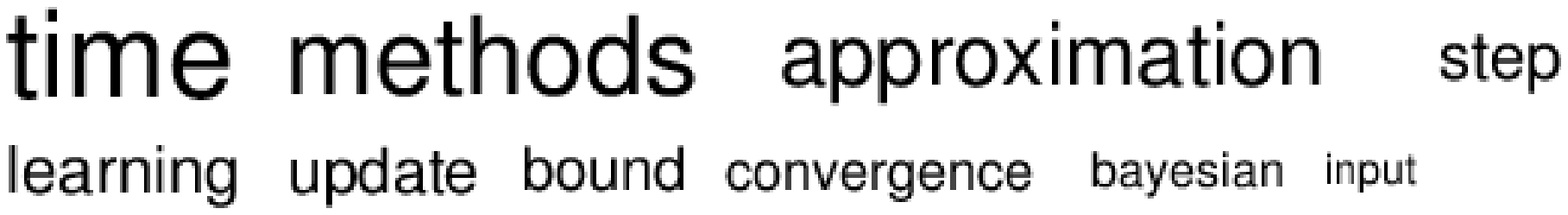}\tabularnewline
\hline 
\end{tabular}\caption{An example of document cluster from NIPS. Top: distribution over authors.
Middle: examples of paper titles. Bottom: examples of word topics
in this cluster.\label{fig:Examples-of-authorNIPS}}
\end{figure}

Our proposed model also allows flexibility in deriving useful understanding
into the data and to evaluate on its predictive capacity (e.g., who
most likely wrote this article, which authors work in the same research
topic and so on). Another possible usage is to obtain \emph{conditional}
distributions among context topics $\phi_{k}$ (s) and content topics
$\psi_{m}$ (s). For example if the context information is timestamp,
the model immediately yields the distribution over time for a topic,
showing when the topic rises and falls. Figure \ref{fig:TimestampContext}
illustrates an example of a distribution over time for a content topic
discovered from PNAS dataset where timestamp was used as context.
This topic appears to capture a congenital disorder known as \emph{Albinism.}
This distribution illustrates research attention to this condition
over the past 100 years from PNAS data. To seek evidence for this
result, we search the term ``Albinism'' in Google Scholar, using
the top 50 searching results and plot the histogram over time in the
same figure. Surprisingly, we obtain a very close match between our
results and the results from Google Scholar as evidenced in the figure.

\begin{figure}
\begin{centering}
\begin{tabular}{l}
\includegraphics[width=0.95\columnwidth]{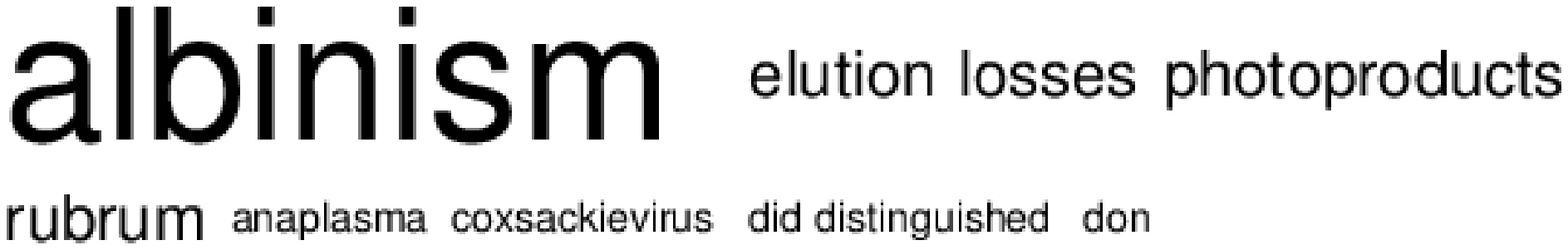}\tabularnewline
\includegraphics[width=0.95\columnwidth]{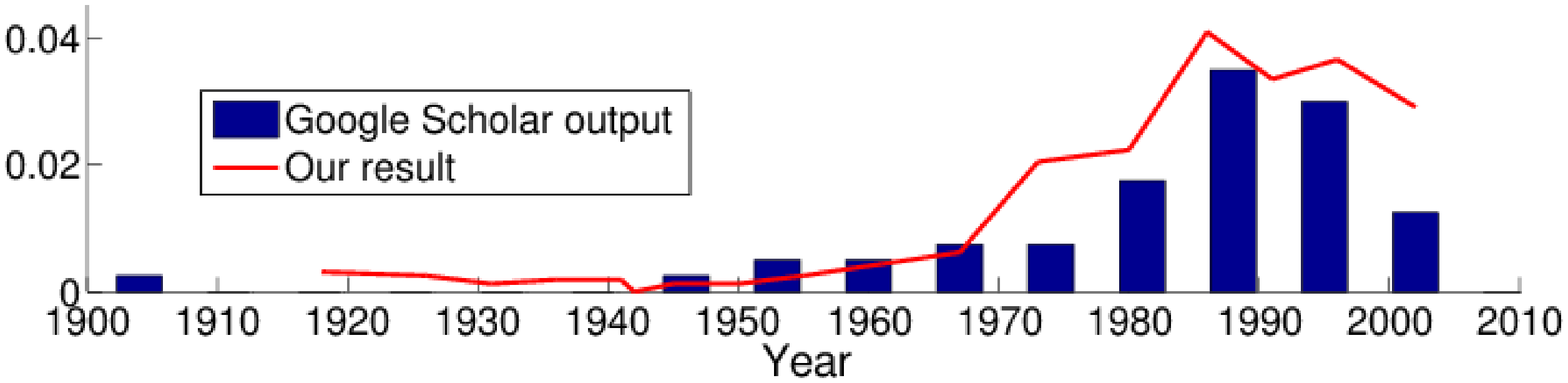}\tabularnewline
\end{tabular}\vspace{-0.2in}

\par\end{centering}

\caption{Topic \emph{Albinism} discovered from PNAS dataset and its conditional
distribution over time using our model; plotted together with results
independently searched from Google Scholar using the top 50 hits.\label{fig:TimestampContext}}
\end{figure}

\paragraph{Image Clustering with Image-Level Tags\protect \\
}

We evaluate the clustering capacity of $\modelname$ using contexts
on an image clustering task. Our dataset is NUS-WIDE described earlier.
We use bag-of-word SIFT features from each image for its content.
Since each image in this dataset comes with a set of tags, we exploit
them as context information, hence each context observation $x_{j}$
is a bag-of-tag annotation vector. 

First we perform the perplexity evaluation for this dataset using
a similar setting as in the previous section. Table \ref{tab:NUS-WIDE-dataset.-Perplexity}
presents the results where our model again outperforms HDP even when
no context (tags) is used for training. 

\begin{table}
\begin{centering}
\begin{tabular}{|c|c|c|}
\hline 
Method & Perplexity & Feature used\tabularnewline
\hline 
\hline 
HDP & $175.62$ & SIFT\tabularnewline
\hline 
$\modelname$ without context & $162.74$ & SIFT\tabularnewline
\hline 
$\modelname$ with context & $\mathbf{152.32}$ & Tags+SIFT\tabularnewline
\hline 
\end{tabular}
\par\end{centering}

\caption{NUS-WIDE dataset. Perplexity is evaluated on SIFT feature.\label{tab:NUS-WIDE-dataset.-Perplexity}}

\end{table}

\begin{figure}
\begin{centering}
\begin{tabular}{cc}
\includegraphics[width=0.45\columnwidth]{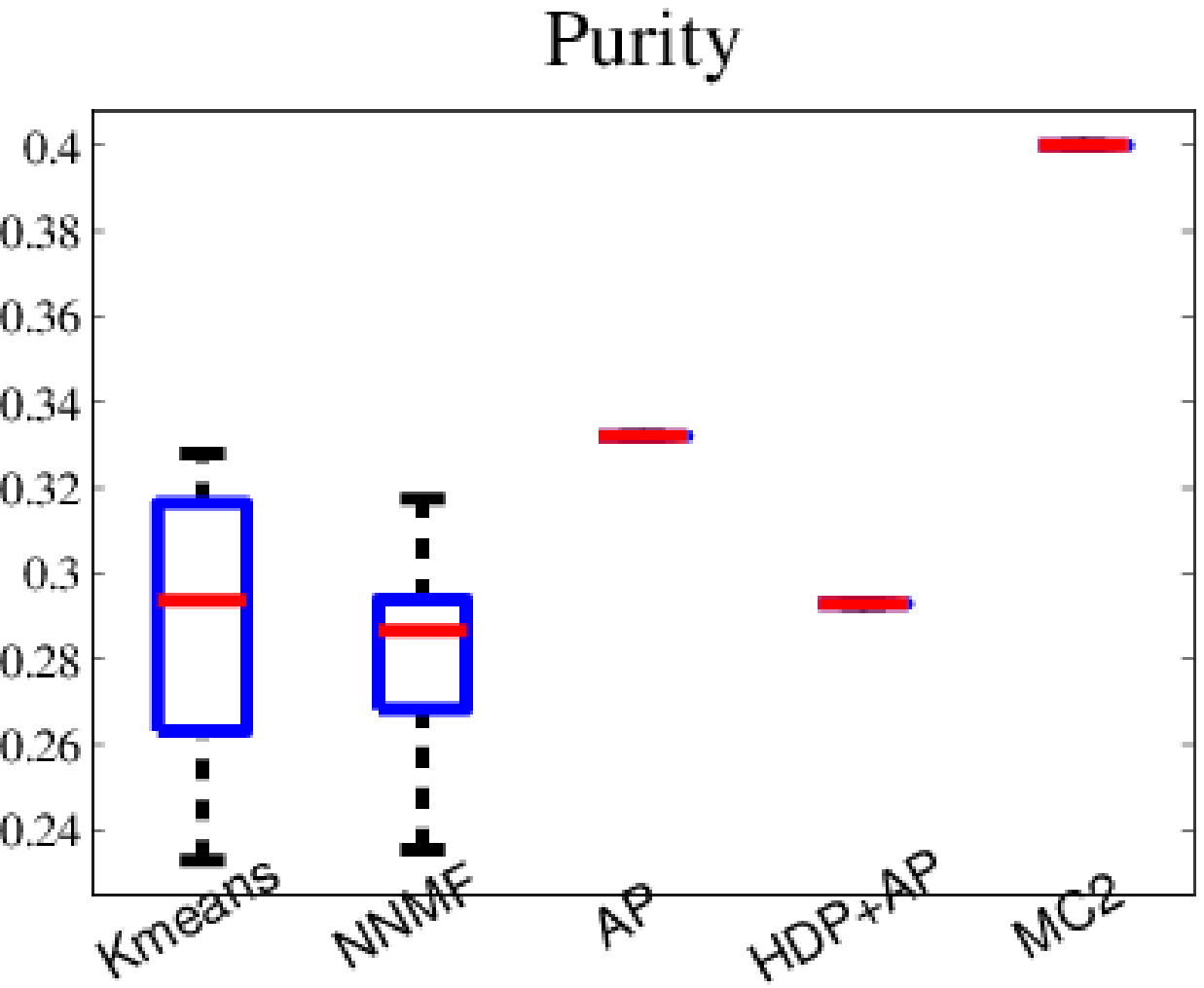} & \includegraphics[width=0.45\columnwidth]{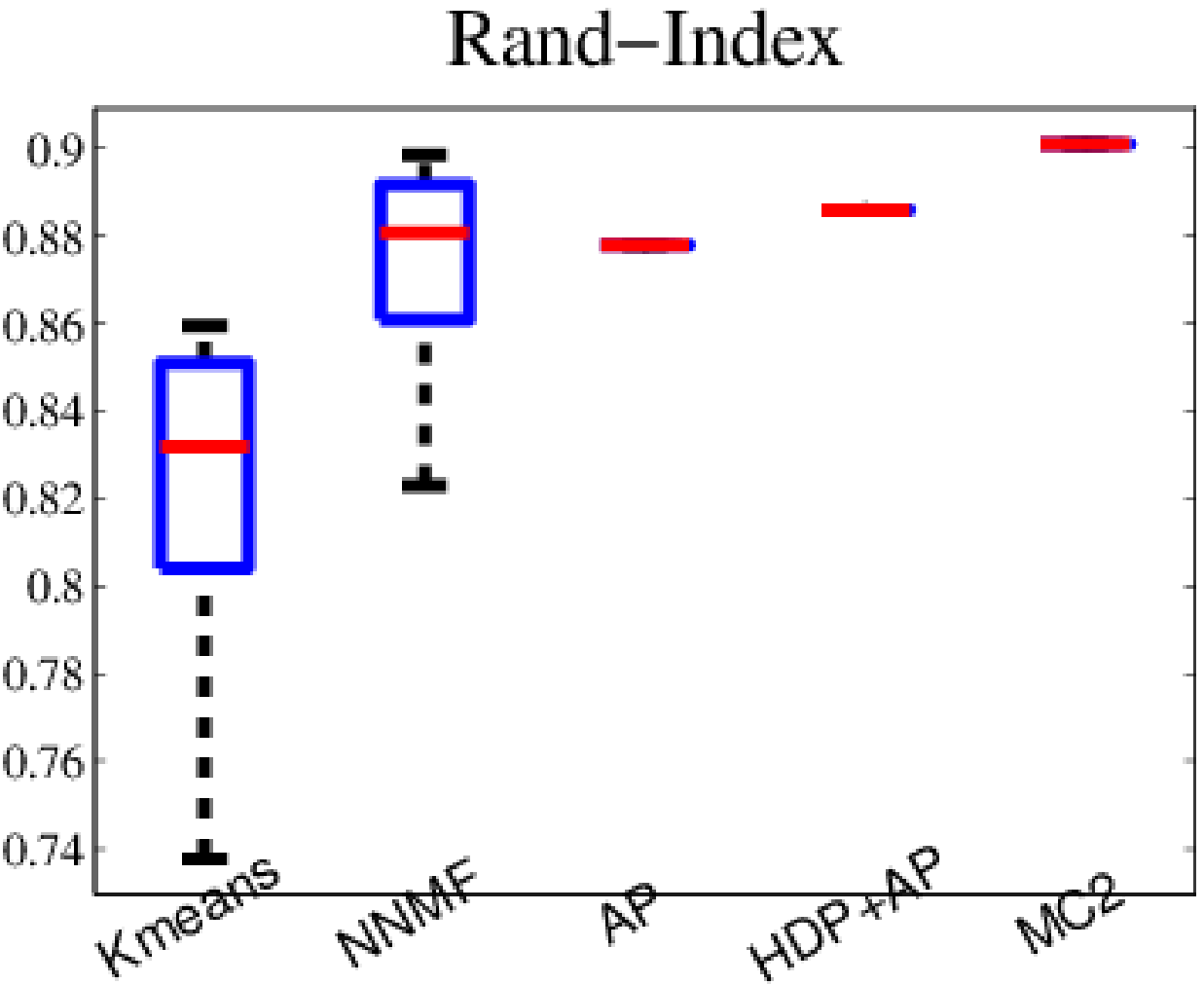}\tabularnewline
\includegraphics[width=0.45\columnwidth]{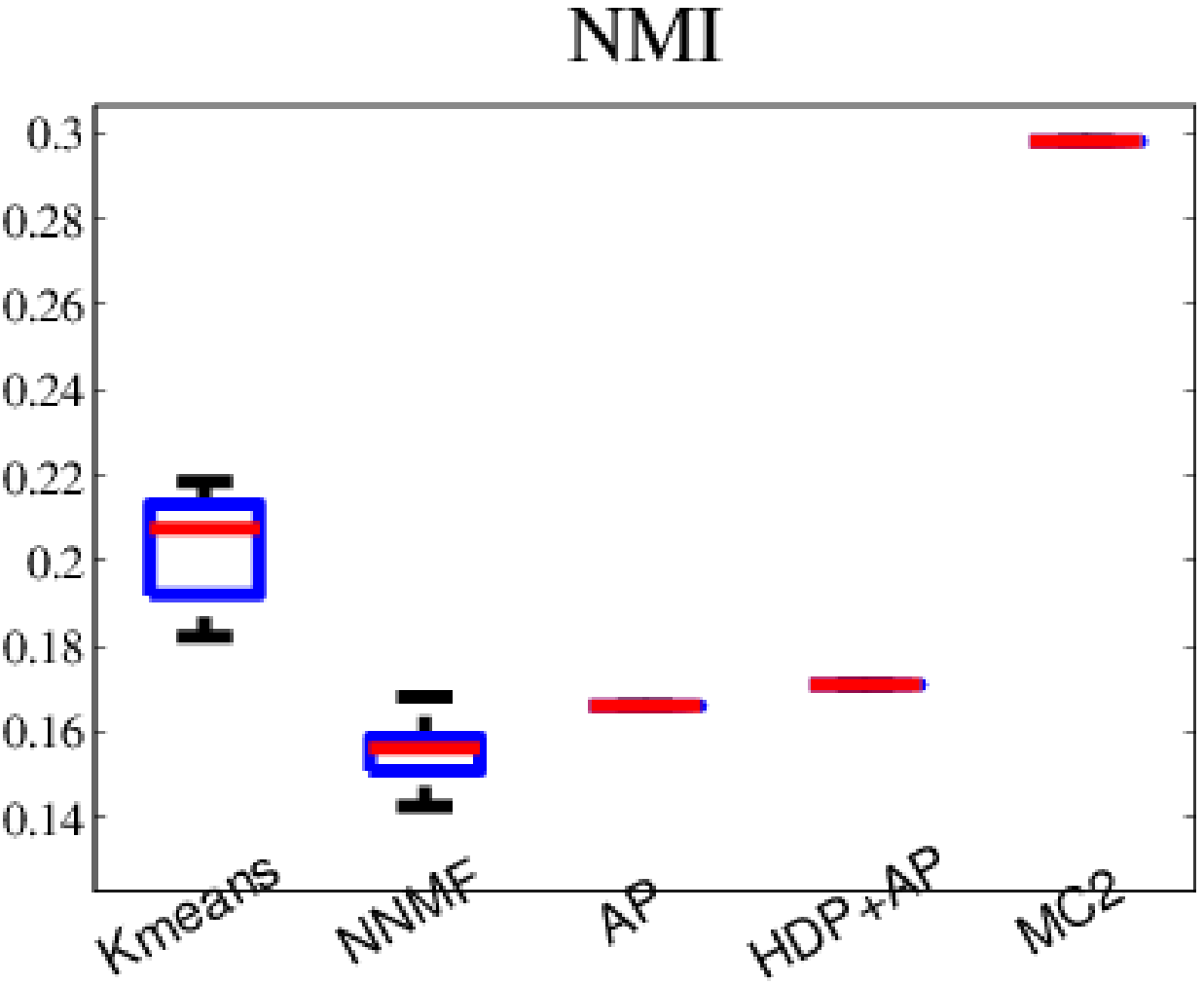} & \includegraphics[width=0.45\columnwidth]{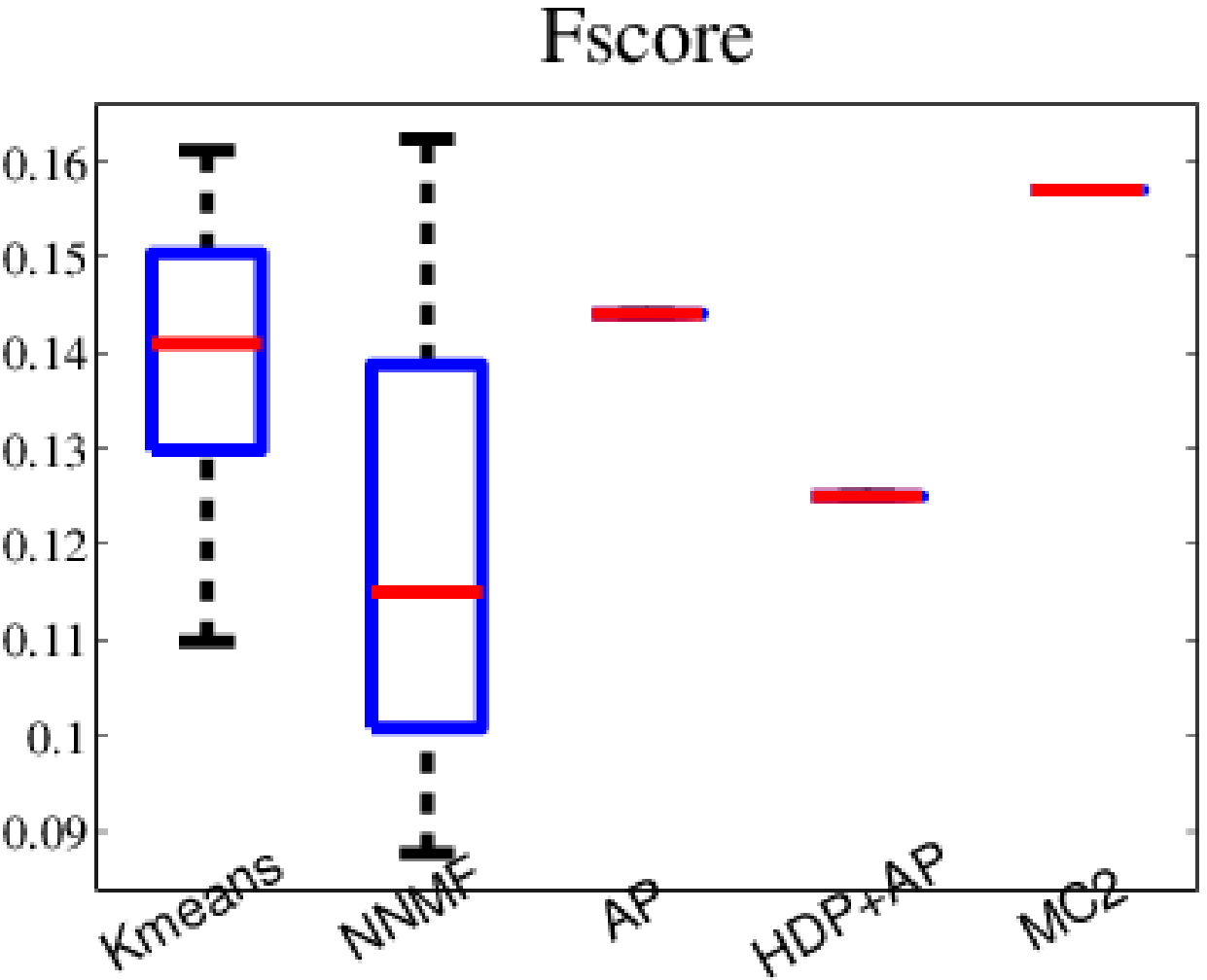}\tabularnewline
\end{tabular}\vspace{-0.2in}
\caption{Clustering performance measured in purity, NMI, Rand-Index and F-score
using NUS-WIDE dataset.\label{fig:NUSWIDE_ClusteringPlot}}

\par\end{centering}

\end{figure}

Next we evaluate the clustering quality of the model using the provided
13 classes as ground truth. We report performance on four well-known
clustering evaluation metrics: Purity, Normalized Mutual Information
(NMI), Rand-Index (RI), and Fscore (detailed in \cite{Rand_71objective,cai2011locally}).
 We use the following baselines for comparison: \vspace{-0.15in}

\begin{itemize}
\item Kmeans and Non-negative Matrix Factorization (NMF)\cite{lee1999learning}.
For these methods, we need to specify the number of clusters in advance,
hence we vary this number from 10 to 40. We then report the min, max,
mean and standard deviation.\vspace{-0.05in}

\item Affinity Propagation (AP) \cite{Frey_Dueck_07clustering}: AP requires
a similarity score between two documents and we use the Euclidean
distance for this purpose.\vspace{-0.05in}

\item Hierarchical Dirichlet Process (HDP) + AP: we first run HDP using
content observations, and then apply Affinity Propagation with similarity
score derived from the symmetric KL divergence between the mixture
proportions from two documents.\vspace{-0.15in}

\end{itemize}
Figure \ref{fig:NUSWIDE_ClusteringPlot} shows the result in which
our model consistently delivers highest performance across all four
metrics. For purity and NMI, our model beats all by a wide margin.

To gain some understanding on the clusters of images induced by our
model, we run t-SNE \cite{van2008visualizing}, projecting the feature
vectors (both content and context) onto a 2D space. For visual clarity,
we randomly select 7 out of 28 clusters and display in Figure \ref{fig:ClusteringTSNE}
where it can be seen that they are reasonably well separated.

\begin{figure}
\begin{centering}
\includegraphics[width=0.95\columnwidth]{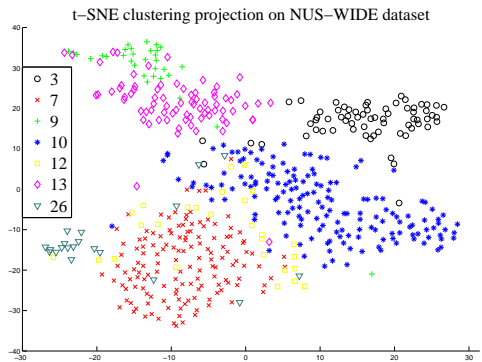}\vspace{-0.3in}

\par\end{centering}

\caption{Projecting 7 discovered clusters (among 28) on 2D using t-SNE \cite{van2008visualizing}.
\label{fig:ClusteringTSNE}}
\end{figure}

\paragraph{Effect of partially observed and missing data\vspace{0.05in}
\protect \\
}

Missing and unlabelled data is commonly encountered in practical applications.
Here we examine the effect of context observability on document clustering
performance. To do so, we again use the NUS-WIDE 13-animal subset
as described previously, then vary the amount of observing context
observation $x_{j}$ with missing proportion ranges from $0\%$ to
$100\%$. 

\begin{table}[th]
\centering{}%
\begin{tabular}{|c|c|c|c|c|}
\hline 
Missing (\%) & Purity & NMI & RI & F-score\tabularnewline
\hline 
0 \% & 0.407 & 0.298 & 0.901 & 0.157\tabularnewline
\hline 
\hline 
25 \% & 0.338 & 0.245 & 0.892 & 0.149\tabularnewline
\hline 
50 \% & 0.320 & 0.236 & 0.883 & 0.137\tabularnewline
\hline 
75 \% & 0.313 & 0.187 & 0.860 & 0.112\tabularnewline
\hline 
100 \% & 0.306 & 0.188 & 0.867 & 0.119\tabularnewline
\hline 
\end{tabular}\vspace{-0.05in}
\caption{Clustering performance with different missing proportion of context
observation $x_{j}$.\label{tab:Clustering_missingdata}}
\end{table}

Table \ref{tab:Clustering_missingdata} reports the result. We make
two observations: a) utilizing context results in a big performance
gain as evidenced in the difference between the top and bottom row
of the table, and b) as the proportion of missing context starts to
increase, the performance degrades gracefully up to 50\% missing.
This demonstrates the robustness of model against the possibility
of missing context data.

\section{Conclusion\label{sec:conclusion}}

We have introduced an approach for multilevel clustering when there
are group-level context information. Our $\modelname$ provides a
single joint model for utilizing group-level contexts to form group
clusters while discovering the shared topics of the group contents
at the same time. We provide a collapsed Gibbs sampling procedure
and perform extensive experiments on three real-world datasets in
both text and image domains. The experimental results using our model
demonstrate the importance of utilizing context information in clustering
both at the content and at the group level. Since similar types of
contexts (time, tags, locations, ages, genres) are commonly encountered
in many real-world data sources, we expect that our model will also
be further applicable in other domains. 

Our model contains a novel ingredient in DP-based Bayesian nonparametric
modeling: we propose to use a base measure in the form of a product
between a context-generating prior $H$ and a content-generating prior
$\DP(vQ_{0})$. Doing this results in a new model with one marginal
being the DPM and another marginal being the nDP mixture, thus establishing
an interesting bridge between the DPM and the nDP. Our product base
measure construction can be generalized to yield new models suitable
for data presenting in more complicated nested group structures (e.g.,
more than 2-level deep).

\bibliographystyle{icml2014}
\bibliography{icml2014,dpresearch}

\appendix

\section{Appendix}

This note provides supplementary information for the main paper. It
has three parts: a) the proof for the marginalization property of
our proposed model, b) detailed derivations for our inference, and
c) equations to show how the perplexity in the experiment was computed.

\subsection{Proof for Marginalization Property (Theorem 4)}

We start with a proposition on the marginalization result for DPM
with the product measure then move on the final proof for our proposed
model.

\subsubsection{Marginalization of DPM with product measure}

Let $H$ be a measure over some measurable space $(\Theta,\Sigma)$.
Let $\mathbb{P}$ be the set of all measures over $(\Theta,\Sigma)$,
suitably endowed with some $\sigma$-algebra. Let $G\sim\DP(\alpha H)$
be a draw from a Dirichlet process.
\begin{lem}
\label{lem:ess_indp}Let $S_{1}\ldots S_{n}$ be $n$ measurable sets
in $\Sigma$. We form a measurable partition of $\Theta$, a collection
of disjoint measurable sets, that generate $S_{1},\ldots,S_{n}$ as
follows. If $S$ is a set, let $S^{1}=S$ and $S^{-1}=\Theta\backslash S$.
Then $S^{*}=\left\{ \bigcap_{i=1}^{n}S_{i}^{c_{i}}\vert c_{i}\in\left\{ 1,-1\right\} \right\} $
is a partition of $\Theta$ into a finite collection of \emph{disjoint}
measurable sets with the property that any $S_{i}$ can be written
as a union of some sets in $S^{*}$. Let the element of $S^{*}$ be
$A_{1}\ldots A_{n^{*}}$ (note $n^{*}\le2^{n})$. Then the expectation
\begin{align}
\ESS[G\left(S_{1}\right),\ldots,G\left(S_{n}\right)]G & =\label{eq:-1}\\
 & \int\prod_{i=1}^{n}G\left(S_{i}\right)\DP\left(dG\gv\alpha H\right)
\end{align}
depends only on $\alpha$ and $H(A_{i})$. In other words, the above
expectation can be written as a function $E_{n}(\alpha,H(A_{1}),\ldots H(A_{n^{*}}))$.
\end{lem}
It is easy to see that since $S_{i}$ can always be expressed as the
sum of some disjoints $A_{i}$, $G\left(S_{i}\right)$ can respectively
be written as the sum of some $G\left(A_{i}\right)$. Furthermore,
by definition of a Dirichlet process, the vector $\left(G\left(A_{1}\right),\ldots,G\left(A_{n^{*}}\right)\right)$
distributed according to a finite Dirichlet distribution $\left(\alpha H\left(A_{1}\right),\ldots,\alpha H\left(A_{n^{*}}\right)\right)$,
therefore the expectation $\ESS[G\left(S_{i}\right)]G$ depends only
on $\alpha$ and $H\left(A_{i}\right)$ (s).
\begin{defn}
(DPM)\label{def:DPM} A DPM is a probability measure over $\Theta^{n}\ni\left(\theta_{1},\ldots,\theta_{n}\right)$
with the usual product sigma algebra $\Sigma^{n}$ such that for every
collection of measurable sets $\left\{ \left(S_{1},\ldots,S_{n}\right):S_{i}\in\Sigma,i=1,\ldots,n\right\} $:
\begin{eqnarray}
\DPM(\theta_{1} & \in & S_{1},\ldots,\theta_{n}\in S_{n}\vert\alpha,H)=\label{eq:}\\
 &  & \int_{G}\prod_{i=1}^{n}G\left(S_{i}\right)\DP\left(dG\gv\alpha H\right)\nonumber 
\end{eqnarray}

\end{defn}
Consider two measurable spaces $(\Theta_{1},\Sigma_{1})$ and $(\Theta_{2},\Sigma_{2})$
and let $(\Theta,\Sigma)$ be their product space where $\Theta=\Theta_{1}\times\Theta_{2}$
and $\Sigma=\Sigma_{1}\times\Sigma_{2}$. We present the general theorem
that states the marginal result from a product base measure.
\begin{prop}
\label{Proof_Prop2}Let $H^{*}$ be a measure over the product space
$\Theta=\Theta_{1}\times\Theta_{2}$. Let $H_{1}$ be the marginal
of $H^{*}$ over $\Theta_{1}$ in the sense that for any measurable
set $A\in\Sigma_{1}$, $H_{1}\left(A\right)=H^{*}\left(A\times\Theta_{2}\right)$.
Denote by $\theta_{i}$ the pair \textup{$\left(\theta_{i}^{\left(1\right)},\theta_{i}^{\left(2\right)}\right)$,}
then:
\begin{align*}
 & \DPM\left(\theta_{1}^{\left(1\right)}\in S_{1},\ldots,\theta_{n}^{\left(1\right)}\in S_{n}\gv\alpha H_{1}\right)\\
 & \quad=\DPM\left(\theta_{1}\in S_{1}\times\Theta_{2},\ldots,\theta_{n}\in S_{n}\times\Theta_{2}\gv\alpha H^{*}\right)
\end{align*}
for every collection of measurable sets $\left\{ \left(S_{1},\ldots,S_{n}\right):S_{i}\in\Sigma_{1},i=1,\ldots,n\right\} $. \end{prop}
\begin{proof}
Since $\left\{ \left(S_{1},\ldots,S_{n}\right):S_{i}\in\Sigma_{1},i=1,\ldots,n\right\} $
are rectangles, expanding the RHS using Definition \ref{def:DPM}
gives:
\[
RHS=\int G\left(S_{1}\times\Theta_{2}\right)\ldots G\left(S_{n}\times\Theta_{2}\right)d\DP(dG\vert\alpha,H^{*})
\]
Let $T_{i}=S_{i}\times\Theta_{2}$, the above expression is the expectation
of $\prod_{i}G(T_{i})$ when $G\sim DP(\alpha H^{*})$. Forming collection
of the disjoint measurable sets $T^{*}=(B_{1}\ldots B_{n^{*}})$ that
generates $T_{i}$, then note that $B_{i}=A_{i}\times\Theta_{2}$,
and $S^{*}=(A_{1}\ldots A_{n^{*}})$ generates $S_{i}$. By definition
of $H_{1}$, $H_{1}(A_{i})=H^{*}(A_{i}\times\Theta_{2})=H^{*}(B_{i})$.
Using the Lemma \ref{lem:ess_indp} above, $RHS=E_{n}(\alpha,H^{*}(B_{1})\ldots H^{*}(B_{n^{*}}))$,
while $LHS=E_{n}(\alpha,H_{1}(A_{1})\ldots H_{1}(A_{n^{*}}))$ and
they are indeed the same. 
\end{proof}
We note that $H^{*}$ can be any arbitrary measure on $\Theta$ and,
in general, we do not require $H^{*}$ to factorize as product measure.

\subsubsection{Marginalization result for our proposed model}

Recall that we are considering a product base-measure of the form
$H^{*}=H\times\DP(vQ_{0})$ for the group-level DP mixture. Drawing
from a DP mixture with this base measure, each realization is a pair
$(\theta_{j},Q_{j})$; $\theta_{j}$ is then used to generate the
context $x_{j}$ and $Q_{j}$ is used to repeatedly generate the set
of content observations $w_{ji}$ within the group $j$. Specifically,
\begin{align}
U & \sim\dirproc\left(\alpha(H\times\DP(vQ_{0}))\right)\quad\text{where }Q_{0}\sim\dirproc\left(\eta S\right)\nonumber \\
(\theta_{j},Q_{j}) & \iid U\ \text{for \ }j=1,\ldots,J\label{eq:MC2}\\
\varphi_{ji} & \iid Q_{j},\quad\text{for each }j\text{ and }i=1,\ldots,N_{j}\nonumber 
\end{align}
In the above, $H$ and $S$ are respectively base measures for context
and content parameters $\theta_{j}$ and $\varphi_{ji}$. We start
with a definition for nested Dirichlet Process Mixture (nDPM) to proceed
further.
\begin{defn}
\label{def:nDPM}(nested DP Mixture) An $\nDPM$ is a probability
measure over $\Theta^{\sum_{j=1}^{J}N_{j}}\ni\left(\varphi_{11},\ldots,\varphi_{1N_{1}},\ldots,\varphi_{JN_{J}}\right)$
equipped with the usual product sigma algebra $\Sigma^{N_{1}}\times\ldots\times\Sigma^{N_{J}}$
such that for every collection of measurable sets $\left\{ \left(S_{ji}\right):S_{ji}\in\Sigma,j=1,\ldots,J,\, i=1\ldots,N_{j}\right\} $:
\begin{align*}
 & \nDPM(\varphi_{ji}\in S_{ji},\forall i,j\vert\alpha,v,\eta,S)\\
 & \,\,\,\,=\int\int\left\{ \prod_{j=1}^{J}\int\prod_{i=1}^{N_{j}}Q_{j}\left(S_{ji}\right)U\left(dQ_{j}\right)\right\} \\
 & \,\,\,\,\,\,\,\,\,\times\DP\left(dU\gv\alpha\DP\left(vQ_{0}\right)\right)\DP\left(dQ_{0}\gv\eta,S\right)
\end{align*}
We now state the main marginalization result for our proposed model.\end{defn}
\begin{thm}
\label{proof_Theorem4}Given $\alpha,H$ and $\alpha,v,\eta,S$, let
$\btheta=\left(\theta_{j}:\forall j\right)$ and $\bvarphi=\left(\varphi_{ji}:\forall j,i\right)$
be generated as in Eq (\ref{eq:MC2}). Then, marginalizing out $\bvarphi$
results in $\DPM\left(\btheta\gv\alpha,H\right)$, whereas marginalizing
out $\btheta$ results in $\nDPM\left(\bvarphi\vert\alpha,v,\eta,S\right)$.\end{thm}
\begin{proof}
First we make observation that if we can show Proposition \ref{Proof_Prop2}
still holds when $H_{1}$ is random with $H_{2}$ is fixed and vice
versa, then the proof required is an immediate corollary of Proposition
\ref{Proof_Prop2} by letting $H^{*}=H_{1}\times H_{2}$ where we
first let $H_{1}=H$, $H_{2}=\DP\left(vQ_{0}\right)$ to obtain the
proof for the first result, and then swap the order $H_{1}=\DP\left(vQ_{0}\right),H_{2}=H$
to get the second result. 

To see that Proposition \ref{Proof_Prop2} still holds when $H_{2}$
is a random measure and $H_{1}$ is fixed, we let the product base
measure $H^{*}=H_{1}\times H_{2}$ and further let $\mu$ be a prior
probability measure for $H_{2}$, i.e, $H_{2}\sim\mu\left(\cdot\right)$.
Denote by $\theta_{i}$ the pair $\left(\theta_{i}^{\left(1\right)},\theta_{i}^{\left(2\right)}\right)$,
consider the marginalization over $H_{2}$:
\begin{align*}
 & \int_{H_{2}}\DPM\left(\theta_{1}\in S_{1}\times\Theta_{2},\ldots,\theta_{n}\in S_{n}\times\Theta_{2}\gv\alpha,H^{*}\right)\mu\left(H_{2}\right)\\
 & \qquad=\int_{\Sigma_{2}}\underbrace{\DPM\left(\theta_{1}^{\left(1\right)}\in S_{1},\ldots,\theta_{n}^{\left(1\right)}\in S_{n}\gv\alpha,H_{1}\right)}_{\text{constant w.r.t }H_{2}}\mu\left(H_{2}\right)\\
 & \qquad=\DPM\left(\theta_{1}^{\left(1\right)}\in S_{1},\ldots,\theta_{n}^{\left(1\right)}\in S_{n}\gv\alpha,H_{1}\right)\int_{\Sigma_{2}}\mu\left(H_{2}\right)\\
 & \qquad=\DPM\left(\theta_{1}^{\left(1\right)}\in S_{1},\ldots,\theta_{n}^{\left(1\right)}\in S_{n}\gv\alpha,H_{1}\right)
\end{align*}
When $H_{1}$ is random and $H_{2}$ is fixed. Let $\lambda\left(\cdot\right)$
be a prior probability measure for $H_{1}$, ie., $H_{1}\sim\lambda\left(\cdot\right)$.
It is clear that Proposition \ref{Proof_Prop2} holds for each draw
$H_{1}$ from $\lambda\left(\cdot\right)$. This complete our proof.
\end{proof}

\subsubsection{Additional result for correlation analysis in nDPM}

We now consider the correlation between $\varphi_{ik}$ and $\varphi_{jk'}$
for arbitrary $i,j,k$ and $k'$, i.e., we need to evaluate:
\begin{align*}
P\left(\varphi_{ik}\in A_{1},\varphi_{jk'}\in A_{2}\gv\alpha,\eta,v,S\right)
\end{align*}
for two measurable sets $A_{1},A_{2}\in\Sigma$ by integrating out
over all immediate random measures. We use an explicit stick-breaking
representation for $U$ where $U\sim\DP\left(\alpha\DP\left(vQ_{0}\right)\right)$
as follows 
\begin{align}
U & =\sum_{k=1}^{\infty}\pi_{k}\delta_{Q_{k}^{*}}\label{eq:U-stick}
\end{align}
where $\pi\sim\gempdf\left(\alpha\right)$ and $Q_{k}^{*}\iid\dirproc\left(vQ_{0}\right)$.
We use the notation $\delta_{Q_{k}^{*}}$ to denote the atomic measure
on measure, placing its mass at measure $Q_{k}^{*}$.

For $i=j$, we have: 
\begin{align*}
P\left(\varphi_{ik}\in A_{1},\varphi_{jk'}\in A_{2}\gv Q_{1},\ldots,Q_{J}\right) & =Q_{i}\left(A_{1}\right)Q_{i}\left(A_{2}\right)
\end{align*}
Sequentially take expectation over $Q_{i}$ and $U$: 
\begin{align*}
\int_{Q_{i}}Q_{i}\left(A_{1}\right)Q_{i}\left(A_{2}\right)dU\left(Q_{i}\right) & =\\
\int_{Q_{i}}Q_{i}\left(A_{1}\right)Q_{i}\left(A_{2}\right)d\left(\sum_{k=1}^{\infty}\pi_{k}\delta_{Q_{k}^{*}}\right) & =\\
\sum_{k}\pi_{k}\left[Q_{k}^{*}\left(A_{1}\right)Q_{k}^{*}\left(A_{2}\right)\right]\\
\int_{U}\sum_{k=1}^{\infty}\pi_{k}\left[Q_{k}^{*}\left(A_{1}\right)Q_{k}^{*}\left(A_{2}\right)\right]d\DP\left(U\gv\alpha\DP\left(vQ_{0}\right)\right) & =\\
\ess\left\{ \sum_{k}\pi_{k}\left[Q_{k}^{*}\left(A_{1}\right)Q_{k}^{*}\left(A_{2}\right)\right]\right\}  & =\\
\sum_{k}\ess\left[\pi_{k}\right]\ess\left[Q_{k}^{*}\left(A_{1}\right)Q_{k}^{*}\left(A_{2}\right)\right] & =\\
\frac{Q_{0}\left(A_{1}\cap A_{2}\right)+Q_{0}\left(A_{1}\right)Q_{0}\left(A_{2}\right)}{v\left(v+1\right)}\left(\sum_{k}\ess\left[\pi_{k}\right]\right) & =\\
\frac{Q_{0}\left(A_{1}\cap A_{2}\right)+Q_{0}\left(A_{1}\right)Q_{0}\left(A_{2}\right)}{v\left(v+1\right)}
\end{align*}
Integrating out $Q_{0}\sim\DP\left(vS\right)$ we get: 

\begin{align*}
P\left(\varphi_{ik}\in A_{1},\varphi_{jk'}\in A_{2}\gv\alpha,v,\eta,S\right) & =\\
\ESS[\frac{Q_{0}\left(A_{1}\cap A_{2}\right)+Q_{0}\left(A_{1}\right)Q_{0}\left(A_{2}\right)}{v\left(v+1\right)}]{Q_{0}\gv\eta,S} & =\\
\frac{1}{v\left(v+1\right)}\left\{ S\left(A_{1}\cap A_{2}\right)+\frac{S\left(A_{1}\cap A_{2}\right)+S\left(A_{1}\right)S\left(A_{2}\right)}{\eta\left(\eta+1\right)}\right\}  & =\\
\frac{S\left(A_{1}\cap A_{2}\right)}{v\left(v+1\right)}+\frac{S\left(A_{1}\cap A_{2}\right)+S\left(A_{1}\right)S\left(A_{2}\right)}{v\left(v+1\right)\eta\left(\eta+1\right)}
\end{align*}
For $i\neq j$, since $Q_{i}$ and $Q_{j}$ are conditionally independent
given $U$, we get:
\begin{align*}
P\left(\varphi_{ik}\in A_{1},\varphi_{jk'}\in A_{2}\gv Q_{1},\ldots,Q_{J}\right) & =Q_{i}\left(A_{1}\right)Q_{j}\left(A_{2}\right)
\end{align*}
Let $a_{k}=Q_{k}^{*}\left(A_{1}\right),b_{k}=Q_{k}^{*}\left(A_{2}\right)$
and using Definition (\ref{def:nDPM}), integrating out $U$ conditional
on $Q_{0}$ with the stick-breaking representation in Eq (\ref{eq:U-stick}):
\begin{align*}
P\left(\varphi_{ik}\in A_{1},\varphi_{jk'}\in A_{2}\gv vQ_{0}\right) & =\\
\left(\int_{U}Q_{i}\left(A_{1}\right)dU\right)\left(\int_{U}Q_{j}\left(A_{2}\right)dU\right) & =\\
\ess\left[\sum_{k}\pi_{k}Q_{k}^{*}\left(A_{1}\right)\right]\left[\sum_{k'}\pi_{k'}Q_{k'}^{*}\left(A_{2}\right)\right] & =\\
\ess\left(\pi_{1}a_{1}+\pi_{2}a_{2}+\ldots\right)\left(\pi_{1}b_{1}+\pi_{2}b_{2}+\ldots\right) & =\\
\ess\left(\sum_{k}\pi_{k}^{2}a_{k}b_{k}\right)+\ess\left(\sum_{k\neq k'}\pi_{k}\pi_{k'}a_{k}b_{k'}\right) & =\\
A\ess\left(\sum_{k}\pi_{k}^{2}\right)+B\ess\left(\sum_{k\neq k'}\pi_{k}\pi_{k'}\right) & =\\
A\sum_{k}\ess\left[\pi_{k}^{2}\right]+B\left(1-\sum_{k}\ess\left[\pi_{k}^{2}\right]\right)
\end{align*}
where
\begin{align*}
A & =\ess\left[a_{k}b_{k}\right]\\
 & =\ess\left[Q_{k}^{*}\left(A_{1}\right)Q_{k}^{*}\left(A_{2}\right)\right]\\
 & =\frac{Q_{0}\left(A_{1}\cap A_{2}\right)+Q_{0}\left(A_{1}\right)Q_{0}\left(A_{2}\right)}{v\left(v+1\right)}
\end{align*}
and since $Q_{k}^{*}$ (s) are iid draw from $\DP\left(vQ_{0}\right)$
we have:
\begin{align*}
B & =\ess\left[a_{k}b_{k'}\right]\\
 & =\ess\left[Q_{k}^{*}\left(A_{1}\right)Q_{k'}^{*}\left(A_{2}\right)\right]\\
 & =\ess\left[Q_{k}^{*}\left(A_{1}\right)\right]\ess\left[Q_{k'}^{*}\left(A_{2}\right)\right]\\
 & =Q_{0}\left(A_{1}\right)Q_{0}\left(A_{2}\right)
\end{align*}
Lastly, since $\left(\pi_{1},\pi_{2},\ldots\right)\sim\gempdf\left(\alpha\right)$,
using the property of its stick-breaking representation $\sum_{k}\ess\left[\pi_{k}^{2}\right]=\frac{1}{1+\alpha}$.
Put things together we obtain the expression for the correlation of
$\varphi_{ik}$ and $\varphi_{jk'}$ for $i\neq j$ conditional on
$Q_{0}$ as:
\begin{align*}
P\left(\varphi_{ik}\in A_{1},\varphi_{jk'}\in A_{2}\gv vQ_{0}\right) & =\\
\frac{Q_{0}\left(A_{1}\cap A_{2}\right)+Q_{0}\left(A_{1}\right)Q_{0}\left(A_{2}\right)}{\left(1+\alpha\right)v\left(v+1\right)}\\
+\frac{\alpha}{1+\alpha}Q_{0}\left(A_{1}\right)Q_{0}\left(A_{2}\right)\\
 & =\\
\frac{Q_{0}\left(A_{1}\cap A_{2}\right)}{\left(1+\alpha\right)v\left(v+1\right)}\\
+\frac{\alpha v\left(v+1\right)+1}{\left(1+\alpha\right)v\left(v+1\right)}Q_{0}\left(A_{1}\right)Q_{0}\left(A_{2}\right)
\end{align*}
 Next, integrating out $Q_{0}\sim\DP\left(vS\right)$ we get:
\begin{align*}
P\left(\varphi_{ik}\in A_{1},\varphi_{jk'}\in A_{2}\gv\alpha,v,\eta,S\right)=\\
\frac{\alpha v\left(v+1\right)+1}{\left(1+\alpha\right)v\left(v+1\right)}\ess\left[Q_{0}\left(A_{1}\right)Q_{0}\left(A_{2}\right)\right]\\
+\frac{\ess\left[Q_{0}\left(A_{1}\cap A_{2}\right)\right]}{\left(1+\alpha\right)v\left(v+1\right)}\\
=\\
\frac{\alpha v\left(v+1\right)+1}{\left(1+\alpha\right)v\left(v+1\right)}\frac{S\left(A_{1}\cap A_{2}\right)+S\left(A_{1}\right)S\left(A_{2}\right)}{\eta\left(\eta+1\right)}\\
+\frac{S\left(A_{1}\cap A_{2}\right)}{\left(1+\alpha\right)v\left(v+1\right)}
\end{align*}

\subsection{Model Inference Derivations\label{sub:Model-Inference-Derivations}}

We provide detailed derivations for model inference with the graphical
model displayed in Fig \ref{fig:J3}. The variables $\phi_{k}$, $\psi_{m}$,
$\pi$, $\tau_{k}$ are integrated out due to conjugacy property.
We need to sample these latent variables $z$, $l$, $\epsilon$ and
hyper parameters $\alpha$, $v$, $\eta$. For convenience of notation,
we denote $z_{-j}$ is a set of latent context variable $z$ in all
documents excluding document $j$, $\bl_{j*}$ is all of hidden variables
$l_{ji}$ in document $j$, and $\bl_{-j*}$ is all of $l$ in other
documents rather than document $j$-th.

\subsubsection*{Sampling $\bz$}

Sampling context index $z_{j}$ needs to take into account the influence
of the corresponding context topics:
\begin{align}
p(z_{j}=k\gv\mathbf{z}_{-j},\bl,\mathbf{x},\alpha,H) & \propto\underbrace{p\left(z_{j}=k\gv\mathbf{z}_{-j},\alpha\right)}_{\text{CRP for context topic}}\label{eq:samplingz}\\
 & \times\underbrace{p\left(x_{j}\gv z_{j}=k,\mathbf{z}_{-j},\mathbf{x}_{-j},H\right)}_{\text{context predictive likelihood}}\nonumber \\
 & \times\underbrace{p\left(\bl_{j*}\gv z_{j}=k,\bl_{-j*},\mathbf{z}_{-j},\epsilon,v\right)}_{\text{content latent marginal likelihood}}\nonumber 
\end{align}
The first term can easily be recognized as a form of Chinese Restaurant
Process (CRP):
\begin{align*}
p\left(z_{j}=k\gv\mathbf{z}_{-j},\alpha\right) & =\begin{cases}
\frac{n_{-j}^{k}}{n_{-j}^{*}+\alpha} & \textrm{if \ensuremath{k}\ensuremath{}old}\\
\frac{\alpha}{n_{-j}^{*}+\alpha} & \textrm{if \ensuremath{k}\ensuremath{}new}
\end{cases}
\end{align*}
where $n_{-j}^{k}$ is the number of data $z_{j}=k$ excluding $z_{j}$,
and $n_{-j}^{*}$ is the count of all \textbf{$\mathbf{z}$}, except
$z_{j}$.

The second expression is the predictive likelihood from the context
observations under the context component $\phi_{k}$. Specifically,
let $f\left(\cdot\gv\phi\right)$ and $h\left(\cdot\right)$ be respectively
the density function for $F\left(\phi\right)$ and $H$, the conjugacy
between $F$ and $H$ allows us to integrate out the mixture component
parameter $\phi_{k}$ , leaving us the conditional density of $x_{j}$
under the mixture component $k$ given all the context data items
exclude $x_{j}$:
\begin{align*}
p\left(x_{j}\gv z_{j}=k,\mathbf{z}_{-j},\mathbf{x}_{-j},H\right) & =\\
\frac{\int_{\phi_{k}}f\left(x_{j}\gv\phi_{k}\right)\underset{j'\neq j,z_{j'}=k}{\prod}f\left(x_{j'}\gv\phi_{k}\right)h\left(\phi_{k}\right)d\phi_{k}}{\int_{\phi_{k}}\underset{j'\neq j,z_{j'}=k}{\prod}f\left(x_{j'}\gv\phi_{k}\right)h\left(\phi_{k}\right)d\phi_{k}}\\
= & f_{k}^{-x_{j}}\left(x_{j}\right)
\end{align*}
Finally, the last term is the contribution from the multiple latent
variables of corresponding topics to that context. Since $l_{ji}\gv z_{j}=k\iid\multpdf\left(\btau_{k}\right)$
where $\btau_{k}\sim\dirpdf\left(v\epsilon_{1},\ldots,v\epsilon_{M},\epsilon_{\text{new}}\right)$,
we shall attempt to integrate out $\btau_{k}$ . Using the Multinomial-Dirichlet
conjugacy property we proceed to compute the last term in Eq (\ref{eq:samplingz})
as following:

\begin{align}
p\left(\bl_{j*}\gv z_{j}=k,\mathbf{z}_{-j},\bl_{-j*},\epsilon,v\right)= & \int_{\btau_{k}}p\left(\bl_{j*}\gv\btau_{k}\right)\label{eq:topic_context1}\\
\times p\left(\btau_{k}\gv\left\{ l_{j'*}\gv z_{j'}=k,j'\neq j\right\} ,\epsilon,v\right)d\btau_{k}\nonumber 
\end{align}
Recognizing the term $p\left(\btau_{k}\gv\left\{ \bl_{j'*}\gv z_{j'}=k,j'\neq j\right\} ,\bepsilon,v\right)$
is a posterior density, it is Dirichlet-distributed with the updated
parameters
\begin{align}
p\left(\btau_{k}\gv\left\{ \bl_{j'*}\gv z_{j'}=k,j'\neq j\right\} \right)\label{eq:topic_context2}\\
=\dirpdf\left(v\epsilon_{1}+c_{k,1}^{-j},\ldots,v\epsilon_{M}+c_{k,M}^{-j},v\epsilon_{\text{new}}\right)\nonumber 
\end{align}
where $c_{k,m}^{-j}=\sum_{j'\neq j}\sum_{i=1}^{N_{j'}}\id\left(l_{j'i}=m,z_{j'}=k\right)$
is the count of topic $m$ being assigned to context $k$ excluding
document $j$. Using this result, $p\left(\bl_{j*}\gv\btau_{k}\right)$
is a predictive likelihood for $\bl_{j*}$ under the posterior Dirichlet
parameters $\btau_{k}$ in Eq \ref{eq:topic_context2} and therefore
can be evaluated to be: 
\begin{align*}
 & p\left(\bl_{j*}\gv z_{j}=k,\mathbf{z}_{-j},\bl_{-j*},\epsilon,v\right)\\
= & \int_{\tau_{k}}p\left(\bl_{j*}\mid\btau_{k}\right)\\
 & \times\textrm{Dir}\left(v\epsilon_{1}+c_{k,1}^{-j},\ldots,v\epsilon_{M}+c_{k,M}^{-j},v\epsilon_{\text{new}}\right)d\btau_{k}\\
= & \int_{\tau_{k}}\prod_{m=1}^{M}\tau_{k,m}^{c_{k,m}^{j}}\times\frac{\Gamma\left(\sum_{m=1}^{M}\left(v\epsilon_{m}+c_{k,m}^{-j}\right)\right)}{\prod_{m=1}^{M}\Gamma\left(v\epsilon_{m}+c_{k,m}^{-j}\right)}\\
 & \times\prod_{m=1}^{M}\tau_{k,m}^{v\epsilon_{m}+c_{k,m}^{-j}-1}d\btau_{k}\\
= & \frac{\Gamma\left(\sum_{m=1}^{M}\left(v\epsilon_{m}+c_{k,m}^{-j}\right)\right)}{\prod_{m=1}^{M}\Gamma\left(v\epsilon_{m}+c_{k,m}^{-j}\right)}\int_{\tau_{k}}\prod_{m=1}^{M}\tau_{k,m}^{v\epsilon_{m}+c_{k,m}^{-j}+c_{k,m}^{j}-1}d\btau_{k}\\
= & \frac{\Gamma\left(\sum_{m=1}^{M}\left(v\epsilon_{m}+c_{k,m}^{-j}\right)\right)}{\prod_{m=1}^{M}\Gamma\left(v\epsilon_{m}+c_{k,m}^{-j}\right)}\\
 & \times\frac{\prod_{m=1}^{M}\Gamma\left(v\epsilon_{m}+c_{k,m}^{-j}+c_{k,m}^{j}\right)}{\Gamma\left(\sum_{m=1}^{M}\left(v\epsilon_{m}+c_{k,m}^{-j}+c_{k,m}^{j}\right)\right)}\\
= & \frac{\Gamma\left(\sum_{m=1}^{M}\left(v\epsilon_{m}+c_{k,m}^{-j}\right)\right)}{\Gamma\left(\sum_{m=1}^{M}\left(v\epsilon_{m}+c_{k,m}^{-j}\right)+N_{j}\right)}\\
 & \times\prod_{m=1}^{M}\frac{\Gamma\left(v\epsilon_{m}+c_{k,m}^{-j}+c_{k,m}^{j}\right)}{\Gamma\left(v\epsilon_{m}+c_{k,m}^{-j}\right)}\\
= & \begin{cases}
A=\frac{\Gamma\left(\sum_{m}\left[v\epsilon_{m}+c_{k,m}^{-j}\right]\right)}{\Gamma\left(\sum_{m}\left[v\epsilon_{m}+c_{k,m}\right]\right)}\prod_{m}\frac{\Gamma\left(v\epsilon_{m}+c_{k,m}\right)}{\Gamma\left(v\epsilon_{m}+c_{k,m}^{-j}\right)} & \text{ if }k\mbox{ \text{old}}\\
B=\frac{\Gamma\left(\sum_{m}v\epsilon_{m}\right)}{\Gamma\left(\sum_{m}v\epsilon_{m}+N_{j}\right)}\prod_{m}\frac{\Gamma\left(v\epsilon_{m}+c_{k,m}^{j}\right)}{\Gamma\left(v\epsilon_{m}\right)} & \text{ if }k\,\textrm{new}
\end{cases}
\end{align*}
note that $\epsilon=\left(\epsilon_{1},\epsilon_{2},...\epsilon_{M},\epsilon_{\textrm{new}}\right)$,
here $\epsilon_{1:M}=\left(\epsilon_{1},\epsilon_{2},...\epsilon_{M}\right)$,
when sampling $z_{j}$ we only use $M$ active components from the
previous iteration. In summary, the conditional distribution to sample
$z_{j}$ is given as:
\begin{align*}
p\left(z_{j}=k\gv\mathbf{z}_{-j},\bl,\mathbf{x},\alpha,H\right)\propto\\
\begin{cases}
n_{-j}^{k}\times f_{k}^{-x_{j}}\left(x_{j}\right)\times A & \text{ if }k\mbox{\text{ previousely used}}\\
\alpha\times f_{k_{\text{new}}}^{-x_{ji}}\left(x_{ji}\right)\times B & \text{ if }k=k_{\text{new}}
\end{cases}
\end{align*}
Implementation note: to evaluate A and B, we make use of the marginal
likelihood resulted from a Multinomial-Dirichlet conjugacy.

\subsubsection*{Sampling $l$}

Let $w_{-ji}$ be the same set as $w$ excluding $w_{ji}$, i.e $w_{-ji}=\left\{ w_{uv}:u\neq j\cap v\neq i\right\} $,
then we can write
\begin{align}
p\left(l_{ji}=m\mid l_{-ji},z_{j}=k,v,w,S\right)\propto\label{eq:sampling_l}\\
\underbrace{p\left(w_{ji}\mid w_{-ji},l_{ji}=m,\rho\right)}_{\textrm{content predictive likelihood}}\times\underbrace{p\left(l_{ji}=m\mid\bl_{-ji},z_{j}=k,\epsilon_{m},v\right)}_{\textrm{CRF for content topic}}\nonumber 
\end{align}
The first argument is computed as log likelihood predictive of the
content with the component $\psi_{m}$
\begin{align}
p\left(w_{ji}\mid w_{-ji},l_{ji}=m,\rho\right)=\label{eq:prob(tji) for prob(lji)}\\
\frac{\int_{\lambda_{m}}s\left(w_{ji}\mid\lambda_{m}\right)\left[\prod_{u\in w_{-ji}(m)}y(u\mid\lambda_{m})\right]s(\lambda_{m})d\lambda_{m}}{\int_{\lambda_{m}}\left[\prod_{u\in w_{-ji}(m)}y\left(u\gv\lambda_{m}\right)\right]s\left(\lambda_{m}\right)d\lambda_{m}}\nonumber \\
\das y_{m}^{-w_{ji}}\left(w_{ji}\right)\nonumber 
\end{align}
And the second term is inspired by Chinese Restaurant Franchise (CRF)
as:
\begin{align}
p\left(l_{ji}=m\mid\bl_{-ji},\epsilon_{m},v\right)= & \begin{cases}
c_{k,m}+v\epsilon_{m} & \textrm{if \ensuremath{m}\ensuremath{}old}\\
v\epsilon_{\textrm{new}} & \textrm{if }m\,\textrm{new}
\end{cases}\label{eq:CRP for prob(lji)}
\end{align}
where $c_{k,m}$ is the number of data point $\left|\left\{ l_{ji}|l_{ji}=m,z_{j}=k,1\le j\le J,1\le i\le N_{j}\right\} \right|$.
The final form to sample $l_{ji}$ is given as:

\begin{align*}
p\left(l_{ji}=m\mid\bl_{-ji},z_{j}=k,w,v,\epsilon\right)\propto\\
\begin{cases}
\left(c_{k,m}+v\epsilon_{m}\right)\times y_{m}^{-w_{ji}}\left(w_{ji}\right) & \textrm{if \ensuremath{m}\ensuremath{}is used previously}\\
v\epsilon_{\textrm{new}}\times y_{m}^{-w_{ji}}(w_{ji}) & \textrm{if }m=m_{\textrm{new}}
\end{cases}
\end{align*}
\textbf{Sampling} $\epsilon$

\begin{flushleft}
Note that sampling $\epsilon$ require both $z$ and $l$.
\par\end{flushleft}

\begin{align}
p\left(\epsilon\mid\bl,\mathbf{z},v,\eta\right) & \propto p\left(\bl\mid\epsilon,v,z,\eta\right)\times p\left(\epsilon\mid\eta\right)\label{eq:sampling epsilon-1}
\end{align}
Isolating the content variables $l_{ji}^{k}$ generated by the same
context $z_{j}=k$ into one group 

$l_{j}^{k}=\left\{ l_{ji}:1\leq i\leq N_{j},z_{j}=k\right\} $ the
first term of \ref{eq:sampling epsilon-1} can be expressed following:
\begin{align*}
p\left(l\mid\epsilon,v,z,\eta\right)= & \prod_{k=1}^{K}\int_{\tau_{k}}p\left(l_{**}^{k}\mid\tau_{k}\right)p\left(\tau_{k}\mid\epsilon\right)d\tau_{k}\\
= & \prod_{k=1}^{K}\frac{\Gamma(v)}{\Gamma\left(v+n_{k*}\right)}\prod_{m=1}^{M}\frac{\Gamma(v\epsilon_{m}+n_{km})}{\Gamma(v\epsilon_{m})}
\end{align*}
where $n_{k*}=\left|\left\{ w_{ji}\mid z_{j}=k,i=1,...N_{j}\right\} \right|$
and $n_{km}=\left|\left\{ w_{ji}\mid z_{j}=k,l_{ji}=m,1\leq j\leq J,1\leq i\leq N_{j},\right\} \right|$.

Let $\eta_{r}=\frac{\eta}{R}$, $\eta_{\text{new}}=\frac{R-M}{R}\eta$
and recall that $\epsilon\sim\dirpdf\left(\eta_{r},\ldots,\eta_{r},\eta_{\text{new}}\right)$,
the last term of Eq \ref{eq:sampling epsilon-1} is a Dirichlet density:
\begin{align*}
p\left(\bepsilon\mid\eta\right)= & \textrm{Dir}\left(\underbrace{\eta_{1},\eta_{2},...\eta_{M}}_{M},\eta_{\textrm{new}}\right)\\
= & \frac{\Gamma(M\times\eta_{r}+\eta_{\textrm{new}})}{[\Gamma(\eta_{r})]^{M}\eta_{\textrm{new}}}\prod_{m=1}^{M}\epsilon_{m}^{\eta_{r}-1}\epsilon_{\textrm{new}}^{\eta_{\textrm{new}}-1}
\end{align*}
Using the result: 
\begin{align*}
\frac{\Gamma(v\epsilon_{m}+n_{km})}{\Gamma(v\epsilon_{m})} & =\sum_{o_{km}=0}^{n_{km}}\textrm{Stirl}\left(o_{km},n_{km}\right)(v\epsilon_{m})^{o_{km}}
\end{align*}
Thus, Eq \ref{eq:sampling epsilon-1} becomes:

\begin{align*}
p\left(\bepsilon\mid\bl,\mathbf{z},v,\eta\right)= & \epsilon_{\textrm{new}}^{\eta_{\textrm{new}}-1}\prod_{k=1}^{K}\frac{\Gamma(v)}{\Gamma\left(v+n_{k*}\right)}\\
\times & \prod_{m=1}^{M}\epsilon_{m}^{\eta_{m}-1}\sum_{o_{km}=0}^{n_{km}}\textrm{Stirl}\left(o_{km},n_{km}\right)(v\epsilon_{m})^{o_{km}}\\
= & \epsilon_{\textrm{new}}^{\eta_{\textrm{new}}-1}\sum_{o_{km}=0}^{n_{km}}\prod_{k=1}^{K}\frac{\Gamma(v)}{\Gamma\left(v+n_{k*}\right)}\\
\times & \prod_{m=1}^{M}\epsilon_{m}^{\eta_{m}-1}\textrm{Stirl}\left(o_{km},n_{km}\right)(v\epsilon_{m})^{o_{km}}\\
p\left(\bepsilon,\bo\mid\bl,\mathbf{z},v,\eta\right)= & \epsilon_{\textrm{new}}^{\eta_{\textrm{new}}-1}\prod_{k=1}^{K}\frac{\Gamma(v)}{\Gamma\left(v+n_{k*}\right)}\\
\times & \prod_{m=1}^{M}\epsilon_{m}^{\eta_{m}-1}\textrm{Stirl}\left(o_{km},n_{km}\right)(v\epsilon_{m})^{o_{km}}
\end{align*}
The probability of the auxiliary variable $o_{km}$ is computed as:

\begin{align*}
p(o_{km}) & =\sum_{o_{km}=0}^{n_{km}}\textrm{Stirl}\left(o_{km},n_{km}\right)(v\epsilon_{m})^{o_{km}}
\end{align*}
Now let $o=\left(o_{km}:\forall k,m\right)$ we derive the following
joint distribution:

\begin{align*}
p\left(\bepsilon\mid o,\bl,\mathbf{z},v,\eta\right) & =\epsilon_{\textrm{new}}^{\eta_{\textrm{new}}-1}\prod_{m=1}^{M}\epsilon_{m}^{\sum_{K}o_{km}+\eta_{m}-1}
\end{align*}
As $R\rightarrow\infty$, we have

\begin{align*}
p\left(\bepsilon\mid\bo,\bl,\mathbf{z},v,\eta\right) & \stackrel{\infty}{=}\epsilon_{\textrm{new}}^{\eta-1}\prod_{m=1}^{M}\epsilon_{m}^{\sum_{K}o_{km}-1}
\end{align*}
Finally, we sample $\epsilon$ jointly with the auxiliary variable
$o_{km}$ by:

\subsection*{
\begin{align*}
p\left(o_{km}=h\gv\cdot\right) & \propto\textrm{Stirl}\left(h,n_{km}\right)(v\epsilon_{m})^{h},\, h=0,1,\ldots,n_{km}\protect\\
p(\epsilon) & \propto\epsilon_{\text{new}}^{\eta-1}\prod_{m=1}^{M}\epsilon_{m}^{\sum_{K}o_{km}-1}
\end{align*}
Sampling hyperparameters}

In the proposed model, there are three hyper-parameters which need
to be sampled : $\alpha,v$ and $\eta$.

\subsubsection*{Sampling $\eta$ }

Using similar strategy and using technique from Escobar and West \cite{Escobar_West_95bayesian},
we have

\begin{align*}
p\left(M\gv\eta,u\right) & =\stir\left(M,u\right)\eta^{M}\frac{\Gamma\left(\eta\right)}{\Gamma\left(\eta+u\right)}
\end{align*}
where $u=\sum_{m}u_{m}$ with $u_{m}=\sum_{K}o_{km}$ is in the previous
sampling $\epsilon$ and $M$ is the number of active content atoms.
Let $\eta\sim\gammapdf\left(\eta_{1},\eta_{2}\right)$. Recall that:
\begin{align*}
\frac{\Gamma\left(\eta\right)}{\Gamma\left(\eta+u\right)} & =\int_{0}^{1}t^{\eta}\left(1-t\right)^{u-1}\left(1+\frac{u}{\eta}\right)dt
\end{align*}
that we have just introduced an auxiliary variable $t$

\begin{align*}
p\left(t\gv\eta\right) & \propto t^{\eta}\left(1-t\right)^{u-1}=\betapdf\left(\eta+1,u\right)
\end{align*}
Therefore,
\begin{align}
p\left(\eta\gv t\right) & \propto\eta^{\eta_{1}-1+M}\exp\left\{ -\eta\eta_{2}\right\} \times t^{\eta}\left(1-t\right)^{u-1}\left(1+\frac{u}{\eta}\right)\nonumber \\
 & =\eta^{\eta_{1}-1+M}\times\exp\left\{ -\eta(\eta_{2}-\log t)\right\} \times\left(1-t\right)^{u-1}\nonumber \\
 & +\eta^{\eta_{1}-1+M-1}\exp\left\{ -\eta(\eta_{2}-\log t)\right\} \times\left(1-t\right)^{u-1}u\nonumber \\
 & \propto\eta^{\eta_{1}-1+M}\exp\left\{ -\eta(\eta_{2}-\log t)\right\} \nonumber \\
 & +u\eta^{\eta_{1}-1+M-1}\exp\left\{ -\eta(\eta_{2}-\log t)\right\} \nonumber \\
 & =\pi_{t}\gammapdf\left(\eta_{1}+M,\eta_{2}-\log t\right)\label{eq:sampling_eta}\\
 & +\left(1-\pi_{t}\right)\gammapdf\left(\eta_{1}+M-1,\eta_{2}-\log t\right)\nonumber 
\end{align}
where $\pi_{t}$ satisfies this following equation to make the above
expression a proper mixture density: 
\begin{align}
\frac{\pi_{t}}{1-\pi_{t}} & =\frac{\eta_{1}+M-1}{u\left(\eta_{2}-\log t\right)}\label{eq:compute_pi_t}
\end{align}
To re-sample $\eta$, we first sample $t\sim\betapdf\left(\eta+1,u\right)$,
compute $\pi_{t}$ as in equation \ref{eq:compute_pi_t}, and then
use $\pi_{t}$ to select the correct Gamma distribution to sample
$\eta$ as in Eq. \ref{eq:sampling_eta}.

\subsubsection*{Sampling $\alpha$ }

Again sampling $\alpha$ is similar to Escobar et al \cite{Escobar_West_95bayesian}.
Assuming $\alpha\sim\gammapdf\left(\alpha_{1},\alpha_{2}\right)$
with the auxiliary variable $t$:

\begin{align*}
p\left(t\mid\alpha,K\right)\propto & t^{\alpha_{1}}\left(1-t\right)^{J-1}\\
p\left(t\mid\alpha,K\right)\propto & \textrm{Beta}\left(\alpha_{1}+1,J\right)
\end{align*}
$J$: number of document

\begin{align*}
p\left(\eta\mid t,K\right)\sim & \pi_{t}\textrm{Gamma}\left(\alpha_{1}+K,\alpha_{2}-\log(t)\right)\\
 & +(1-\pi_{t})\textrm{Gamma}\left(\alpha_{1}+K-1,\alpha_{2}-\log(t)\right)
\end{align*}
where $c,d$ are prior parameter for sampling $\eta$ following Gamma
distribution and $\frac{\pi_{t}}{1-\pi_{t}}=\frac{\alpha_{1}+K-1}{J\left(\alpha_{2}-\log t\right)}$

\subsubsection*{Sampling $v$}

Sampling $v$ is similar to sampling concentration parameter in HDP
\cite{Teh_etal_06hierarchical}. Denote $o_{k*}=\sum_{m}o_{km}$,
where $o_{km}$ is defined previously during the sampling step for
$\bepsilon$, $n_{k*}=\sum_{m}n_{km}$, where $n_{km}$ is the count
of $\left|\left\{ l_{ji}\mid z_{ji}=k,l_{ji}=m\right\} \right|$.
Using similar technique in \cite{Teh_etal_06hierarchical}, we write:

\begin{align*}
p\left(o_{1*},o_{2*}..,o_{K*}\mid v,n_{1*},...n_{K*}\right)= & \prod_{k=1}^{K}\stir(n_{k*},o_{k*})\alpha_{0}^{o_{k*}}\\
 & \times\frac{\Gamma(v)}{\Gamma\left(v+n_{k*}\right)}
\end{align*}
where the last term can be expressed as

\begin{align*}
\frac{\Gamma(v)}{\Gamma\left(v+n_{k*}\right)}= & \frac{1}{\Gamma(n_{k*})}\int_{0}^{1}b_{k}^{v}\left(1-b_{k}\right)^{n_{k*}-1}\left(1+\frac{n_{k*}}{v}\right)db_{k}
\end{align*}
Assuming $v\sim\gammapdf\left(v_{1},v_{2}\right)$, define the auxiliary
variables $b=\left(b_{k}\mid k=1,\ldots,K\right),b_{k}\in[0,1]$ and
$t=\left(t_{k}\mid k=1,\ldots,K\right),t_{k}\in\left\{ 0,1\right\} $
we have

\begin{align*}
q\left(v,b,t\right) & \propto v^{v_{1}-1+\sum_{k}M_{k}}\exp\left\{ -vv_{1}\right\} \\
 & \times\prod_{k=1}^{K}b_{k}^{v}\left(1-b_{k}\right)^{M_{k}-1}\left(\frac{M_{k}}{v}\right)^{t_{k}}
\end{align*}
We will sample the auxiliary variables $b_{k}$, $t_{k}$ in accordance
with $v$ that are defined below:

\begin{align*}
q(b_{k}\mid v)= & \textrm{Beta}\left(v+1,o_{k*}\right)\\
q\left(t_{k}\mid.\right)= & \textrm{Bernoulli}\left(\frac{o_{k*}/v}{1+o_{k*}/v}\right)\\
q(v\mid.)= & \textrm{Gamma}\left(v_{1}+\sum_{k}\left(o_{k*}-t_{k}\right),v_{2}-\sum_{k}\log b_{k}\right)
\end{align*}

\subsection{Relative Roles of Context and Content Data\label{sub:Relative-Roles-of}}

Regarding the inference of the cluster index $z_{j}$ (Eq. \ref{eq:samplingz}),
to obtain the marginal likelihood (the third term in Eq. \ref{eq:samplingz})
one has to integrate out the words' topic labels $l_{ji}$. In doing
so, it can be shown that the \emph{sufficient} statistics coming from
the content data toward the inference of the topic frequencies and
the clustering labels will just be the empirical word frequency from
each document. As each document becomes sufficiently long, the empirical
word frequency quickly concentrates around its mean by the central
limit theorem (CLT), so as soon as the effect of CLT kicks in, increasing
document length further will do very little in improving this sufficient
statistics. 

Increasing the document length will probably not hurt, of course.
But to what extent it contributes relative to the number of documents
awaits a longer and richer story to be told. 

We confirm this argument by varying the document length and the number
of documents in the synthetic document and see how they affect the
\emph{posterior} of the clustering labels. Each experiment is repeated
20 times. We record the mean and standard deviation of the clustering
performance by NMI score. As can be seen from Fig \ref{fig:ClusteringPerformance_VaryingDocumentLength},
using context observation makes the model more robust in recovering
the true document clusters.

\begin{figure*}
\begin{centering}
\includegraphics[width=1.8\columnwidth]{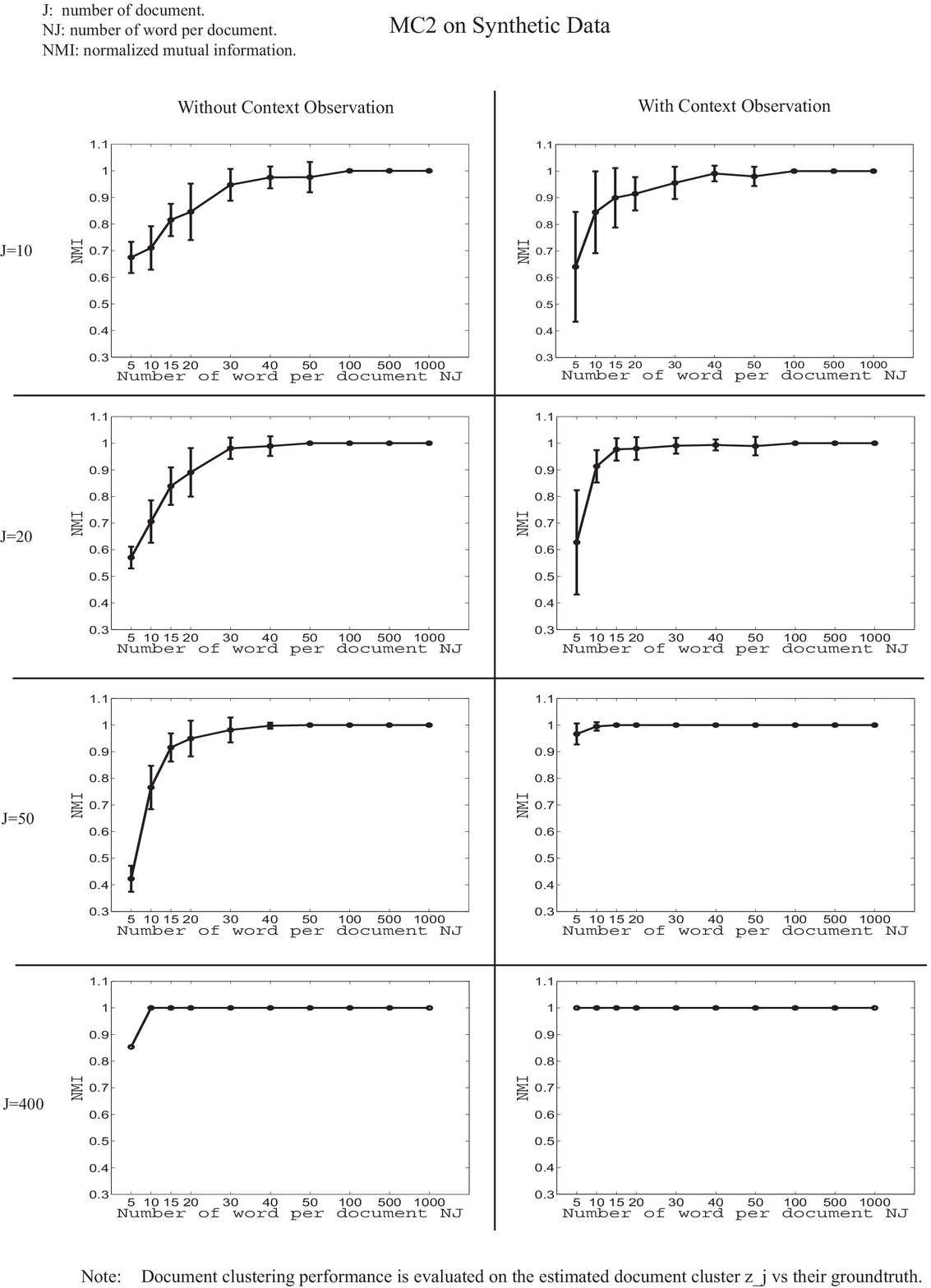}
\par\end{centering}

\caption{Document clustering performance with different numbers of observed
words and documents.\label{fig:ClusteringPerformance_VaryingDocumentLength}}

\end{figure*}

\subsection{Perplexity Evaluation\label{sub:Perplexity-Evaluation}}

The standard perplexity proposed by Blei et al \cite{Blei_etal_03},
used to evaluate the proposed model as following:

\begin{align*}
\textrm{perplexity}\left(w^{\textrm{Test}}\right) & =\exp\left\{ -\frac{\sum_{j=1}^{J_{\textrm{Test}}}\log p\left(w_{j}^{\textrm{Test}}\right)}{\sum_{j=1}^{J_{\textrm{Test}}}N_{j}^{\textrm{Test}}}\right\} 
\end{align*}
During individual sampling iteration $t$, we utilize the important
sampling approach \cite{teh2006collapsed} to compute $p\left(w_{\textrm{Test}}\right)$.
The posterior estimation of $\psi_{m}$ in a Multinomial-Dirichlet
case is defined below, note that it can be in other types of conjugacies
\cite{gelman2003bayesian} (e.g. Gaussian-Wishart, Binomial-Poisson):

\begin{align*}
\psi_{m,v}^{t} & =\frac{n_{m,v}^{t}+\textrm{smooth}}{\sum_{u=1}^{V}n_{m,v}^{t}+V\times\textrm{smooth}}
\end{align*}

\begin{align*}
\tau_{k,m}^{t} & =\frac{c_{k,m}+vv\times\epsilon_{m}}{\sum_{m=1}^{M}\left(c_{k,m}+vv\times\epsilon_{m}\right)}
\end{align*}
where $n_{m,v}^{t}$ is number of times a word $v$, $v\in\left\{ 1,...,V\right\} $
is assigned to context topic $\psi_{m}$ in iteration $t$, and $c_{k,m}$
is the count of the set $\left\{ w_{ji}\mid z_{j}=k,l_{ji}=m,0\le j\le J,0\le i\le N_{j}\right\} $.
There is a constant smooth parameter \cite{AsuWelSmy2009a} that influence
on the count, roughly set as 0.1. Supposed that we estimate $z_{j}^{\textrm{Test}}=k$
and $l_{ji}^{\textrm{Test}}=m$, then the probability $p\left(w_{j}^{\textrm{Test}}\right)$
is computed as:

\begin{align*}
p\left(w_{j}^{\textrm{Test}}\right)=\prod_{i=1}^{N_{j}^{\textrm{Test}}}\frac{1}{T}\sum_{t=1}^{T} & \tau_{k,m}^{t}\psi_{m,w_{ji}^{\textrm{Test}}}^{t}
\end{align*}
where T is the number of collected Gibbs samples.

\end{document}